\documentclass[final,12pt]{colt2023} % Anonymized submission
\title[A Unified Analysis of Nonstochastic 
 Delayed Feedback]{A Unified Analysis of Nonstochastic 
 Delayed Feedback for Combinatorial Semi-Bandits, Linear Bandits, and MDPs}
\usepackage{times}
\usepackage{bm}
\usepackage{bbm}

\usepackage{enumitem}
\usepackage{comment}

\usepackage{thmtools, thm-restate}
\usepackage{color} 
\usepackage{float}
\usepackage{dsfont}
\usepackage{amsfonts}

\usepackage{natbib}

\usepackage{nicefrac} 
\usepackage{mathtools}
\usepackage{algorithmic}
\usepackage{mathrsfs}
\DeclareBoldMathCommand{\0}{0}
\DeclareBoldMathCommand{\1}{1}
\DeclareBoldMathCommand{\a}{a}
\DeclareBoldMathCommand{\b}{b}
\DeclareBoldMathCommand{\bell}{\ell}
\DeclareBoldMathCommand{\e}{e}
\DeclareBoldMathCommand{\h}{h}
\DeclareBoldMathCommand{\h}{h}
\DeclareBoldMathCommand{\I}{I}
\DeclareBoldMathCommand{\L}{L}
\DeclareBoldMathCommand{\p}{p}
\DeclareBoldMathCommand{\q}{q}
\DeclareBoldMathCommand{\u}{u}
\DeclareBoldMathCommand{\v}{v}
\DeclareBoldMathCommand{\w}{w}
\DeclareBoldMathCommand{\x}{x}
\DeclareBoldMathCommand{\y}{y}
\DeclareBoldMathCommand{\z}{z}

\renewcommand{\hat}{\widehat}
\renewcommand{\tilde}{\widetilde}
\renewcommand{\bar}{\overline}

\newcommand{\action}{\a}
\newcommand{\actiondim}{K}
\newcommand{\actionset}{\mathcal{A}}

\newcommand{\ball}{\mathcal{B}}
\newcommand{\bbound}{\beta}
\newcommand{\blhat}{\hat{\bell}}
\newcommand{\blhatt}{\hat{\bell}_t}

\newcommand{\br}{B}
\newcommand{\combm}{B}
\newcommand{\dikin}{\mathcal{D}}
\newcommand{\dmax}{d_{\max}}

\newcommand{\domainw}{\mathcal{W}}

\newcommand{\fbound}{\alpha}
\newcommand{\fil}{\mathcal{F}}
\newcommand{\filobs}{\fil}
\newcommand{\half}{\tfrac{1}{2}}
\newcommand{\indicator}{\mathbb{I}}
\newcommand{\Lhat}{\hat{\L}}

\newcommand{\Lstar}{\hat{\L}^\star}
\newcommand{\miss}{m}
\newcommand{\nats}{\mathbb{N}}

\newcommand{\obs}{o}
\newcommand{\reals}{\mathbb{R}}
\newcommand{\regret}{\mathcal{R}}
\newcommand{\sphere}{\mathcal{S}}

\newcommand{\sumT}{\sum_{t=1}^T}

\newcommand{\calM}{\mathcal{M}}
\newcommand{\calP}{\mathcal{P}}
\newcommand{\calF}{\mathcal{F}}
\newcommand{\calG}{\mathcal{G}}
\newcommand{\sinit}{s_{\text{init}}}
\newcommand{\bbI}{\mathbb{I}}
\newcommand{\bbE}{\mathbb{E}}
\newcommand{\bbR}{\mathbb{R}}
\renewcommand{\l}{\left}
\renewcommand{\r}{\right}
\newcommand{\bias}{\b}

\newcommand{\totaldelay}{D}
\newcommand{\combslider}{\theta}
\newcommand{\observe}{\mathcal{L}}
\newcommand{\Aset}{\mathcal{A}}
\newcommand{\mdpA}{\mathscr{A}}
\newcommand{\mdpS}{\mathscr{S}}
\newcommand{\mdpell}{\bell}

\newcommand{\hinta}{$H_1$}
\newcommand{\hintb}{$H_2$}
\newcommand{\hintc}{$H_3$}

\newcommand{\Conv}{\mathrm{Conv}}

\newcommand{\sumdim}{\sum_{i = 1}^\actiondim}

\DeclareMathOperator{\E}{\mathbb E}
\DeclareMathOperator*{\argmin}{arg\,min}

\newcommand\numberthis{\addtocounter{equation}{1}\tag{\theequation}}

\DeclareRobustCommand{\VAN}[3]{#2} % proper Dutch 'van/de' capitalisation

\newcommand{\TODO}[1]{%
\ifmmode
\text{\textcolor{red}{TODO: #1}}
\else
\textcolor{red}{TODO: #1}
\fi
}

\coltauthor{%
 \Name{Dirk van der Hoeven} \Email{dirk@dirkvanderhoeven.com}\\
 \addr Korteweg-de Vries Institute for Mathematics, University of Amsterdam, The Netherlands 
 \AND
 \Name{Lukas Zierahn} \Email{lukaszierahn@gmail.com}\\
 \addr Università degli Studi di Milano, Italy
  \AND
 \Name{Tal Lancewicki} \Email{lancewicki@mail.tau.ac.il}\\
 \addr Blavatnik School of Computer Science, Tel Aviv University, Israel
  \AND
 \Name{Aviv Rosenberg} \Email{avivros007@gmail.com}\\
 \addr Amazon Science
  \AND
 \Name{Nicol{\'o} Cesa-Bianchi} \Email{nicolo.cesa-bianchi@unimi.it}\\
 \addr Università degli Studi di Milano and Politecnico di Milano, Italy
}

\begin{document}

\maketitle

\begin{abstract}%
  % We study online learning with delayed bandit feedback. We provide a new analysis of Follow The Regularized Leader, which separates the cost of delayed feedback and bandit feedback. We then apply our analysis to several applications and obtain novel results. In the combinatorial semi-bandit setting we apply this new analysis to derive the first regret bounds in this setting, which we show is optimal. We then turn to Markov decision processes with known transitions functions, where we provide the first optimal regret bounds. Finally, we provide an efficient algorithm for linear bandits, obtains near-optimal regret bounds.
  We derive a new analysis of Follow The Regularized Leader (FTRL) for online learning with delayed bandit feedback. By separating the cost of delayed feedback from that of bandit feedback, our analysis allows us to obtain new results in three important settings. On the one hand, we derive the first optimal (up to logarithmic factors) regret bounds for combinatorial semi-bandits with delay and adversarial Markov decision processes with delay (and known transition functions). 
  On the other hand, we use our analysis to derive an efficient algorithm for linear bandits with delay achieving near-optimal regret bounds. Our novel regret decomposition shows that FTRL remains stable across multiple rounds under mild assumptions on the Hessian of the regularizer.
\end{abstract}


\section{Introduction}
Delayed feedback is a phenomenon that cannot be avoided in many applications of online learning. For example, in digital advertisement a conversion event may happen with some delay after an ad is shown to a user. In healthcare, the effect of a drug on a patient may take some time before it becomes observable \citep{Eic88}. A consequence of delayed feedback is that sequential decision makers have to act before knowing the effect of their previous actions, where the effect of multiple past actions may be potentially observed all at once.
These challenges pertain not only to the algorithms, but also to the way they are analyzed, which is the reason why standard (non-delayed) proof techniques fail in the presence of delayed feedback.

Due to its fundamental nature in online learning, delayed feedback has been extensively studied in several different scenarios, including full-information feedback \citep{weinberger2002delayed, joulani2013online, quanrud2015online, joulani2016delay, flaspohler2021online} and bandit feedback \citep{cesa2016delay,thune2019nonstochastic,bistritz2019online,zimmert2020optimal,ito2020delay,gyorgy2020adapting,van2021nonstochastic,masoudian2022best}.
In this work, we focus on the more realistic case of bandit feedback; that is, when the only way for the learner to know the effect of an action is to execute it. We develop a general framework for the analysis of delayed bandit feedback which we then apply to three important settings: combinatorial semi-bandits (which includes multi-armed bandits as a special case), adversarial Markov Decision Processes (MDPs), and linear bandits. 

Our analysis, which is based on Follow The Regularized Leader (FTRL)---see, for example, \citep[Chapter~7]{orabona2019modern}, unifies previous analyses and sheds light on the impact of delayed bandit feedback in online learning. Our main insight is that one can separate the cost of delayed feedback and bandit feedback through a novel decomposition of the FTRL regret, which allows to separately bound these different regret components. This insight leads to new results in all of the settings we consider. We prove the first regret bounds for combinatorial semi-bandits with delays, which also turn out to be optimal for sufficiently large $T$ (throughout the paper, by optimal we always mean optimal for sufficiently large $T$). We also prove the first regret bounds for adversarial MDPs with delays and known transitions, which are again optimal. Finally, we derive a computationally efficient algorithm for linear bandits, whose regret has an optimal dependence on delays.
% has a slightly suboptimal regret but handles delays optimally. 

We now formally introduce the setting of online learning with delayed bandit feedback studied in this paper. Online learning with delayed bandit feedback proceeds in rounds. In each round $t \in [T]$ the learner chooses (possibly in a randomized manner) an action $\action_t \in \Aset \subseteq \reals^K$, suffers loss $\action_t^\top\bell_t$, where $\bell_t \in \reals^K$ is bounded in some suitably chosen norm, and observes $\{\observe(\bell_{\tau}, \action_{\tau}): \tau + d_\tau = t\}$, where $d_1,\ldots,d_T$ is an unknown sequence of delays and $\observe$ is an application-specific (possibly randomized) feedback function, encoding which information about $\bell_{\tau}$ the learner sees based on the action $\action_\tau$. 
%  and $\observe$ is an application-specific feedback function, encoding which coordinates of $\bell_{\tau}$ the learner sees based on the action $\action_\tau$. 
For example, in the combinatorial semi-bandit setting the learner observes all loss components corresponding to the non-zero elements of the action, whereas in the linear bandit setting the learner only observes the scalar $\action_\tau^\top\bell_\tau$.

\subsection{Contributions}

%Here we describe our contributions to the field. 

\paragraph{New analysis.}
In section~\ref{sec:analysis} we provide a novel analysis of FTRL under delayed bandit feedback. The main novelty is showing that we can decompose the regret into three main parts. The first part of the regret is standard, namely the pseudo-distance between the starting point of the algorithm and the optimal point in hindsight. The second part is the cost of delayed feedback. In our analysis, we show that the cost of delayed feedback is essentially the same as in the delayed full-information setting. The third part of the regret is the cost of bandit feedback, which is the same term that occurs in the standard analysis of FTRL for bandit feedback. A technical novelty is that we show that FTRL is stable across multiple rounds under some mild assumptions on the Hessian of the regularizer. In related work, \citet{huang2023banker} provides an analysis of online mirror descent with delayed bandit feedback in several settings. However, their analysis does not lead to optimal bounds because it does not separate the cost of delayed and bandit feedback.

\paragraph{Combinatorial semi-bandits with delayed feedback.} 
As far as we know we are the first to consider nonstochastic combinatorial semi-bandits under delayed feedback. In the combinatorial semi-bandit setting, we apply the newly gained insight from our analysis of FTRL to derive an optimal algorithm. We show that if $\max_{\action\in \actionset}\|\action\|_1 \leq \combm$, then the regret after $T$ rounds is of order $\sqrt{\combm(KT + \combm \totaldelay)\log(K)}$, where $D = \sumT d_t$ is the total delay after $T$ rounds. In the worst case, the delay is constant (i.e., $d_t = d$ for all $t$) and we provide a matching lower bound (up to logarithmic factors) showing that any learner must incur $\Omega(\sqrt{\combm T (K + \combm d)})$ regret.

\paragraph{Adversarial Markov decision processes.} 
Delayed feedback in adversarial (finite-horizon and episodic) MDPs was first studied by \cite{lancewicki2020learning}. Under full-information feedback, where the agent observes the entire cost function at the end of the episode, they achieve optimal regret $\tilde O(H \sqrt{T+D})$, where $T$ is the number of episodes and $H$ is the horizon. However, under the more realistic bandit feedback (where the only observed costs are those along the agent's trajectory), their regret bound scales with $T^{2/3} + D^{2/3}$, which is far from optimal.
The current state-of-the-art guarantee under bandit feedback is by \cite{jin2022near} who achieve regret bound of $\tilde O(H\sqrt{S A T} + H (H S A)^{1/4}\sqrt{D})$. However, there is still a $(HSA)^{1/4}$ factor gap on the second term compared to the lower bound of \citet{lancewicki2020learning}. Remarkably, the application of our FTRL analysis to adversarial MDPs allows us to close this gap and achieve the first optimal regret bound of $\tilde O(H\sqrt{S A T} + H \sqrt{D})$ for the case of known transitions.

\paragraph{Linear bandits.}
In the linear bandit setting, \citet{ito2020delay} provide an analysis of continuous exponential weights \citep{cover1991universal, vovk1990aggregating, littlestone1994weighted} with delayed bandit feedback and constant delay $d$ that obtains the optimal $\tilde O(K\sqrt{T} + \sqrt{dT})$ regret bound. One drawback is that the per-round runtime of continuous exponential weights is prohibitively large, although it is polynomial in $K$ and $T$. Building on Scrible \citep{abernethy2008competing}, we derive an algorithm that achieves a slightly suboptimal $\tilde O(K^{3/2}\sqrt{T} + \sqrt{D})$ regret, but with a much better per-round running time of order $K^3$, provided a self-concordant barrier for the decision set can be efficiently computed. \citet{huang2023banker} show an algorithm with a similar running time, but with a worse regret bound of  $\tilde{O}(K^{3/2}\sqrt{T} + K^2\sqrt{D})$. 

\subsection{Additional related work}

%Here we discuss the relevant literature that we have not yet discussed. 

\paragraph{Delayed feedback in stochastic models.}
Delayed feedback with stochastic losses where studied both in MDPs \citep{howson2021delayed} and linear bandits \citep{vernade2020linear,howson2022delayed}, as well as many other domains \citep{dudik2011efficient,agarwal2012distributed,vernade2017stochastic,vernade2020linear,pike2018bandits,cesa2018nonstochastic,zhou2019learning,manegueu2020stochastic,lancewicki2021stochastic,cohen2021asynchronous}. However, the adversarial losses considered in this work are much more general and induce many additional technical challenges.

\paragraph{Combinatorial semi-bandits with delayed feedback.} Even though we are the first to study combinatorial semi-bandits with delayed feedback, a special case, namely multi-armed bandits with delayed feedback, is well understood. \citet{neu2010online, neu2014bandit} were among the first ones to study the impact of delayed feedback in the nonstochastic setting. Subsequently, \citet{cesa2019delay} proved a $\Omega(\sqrt{KT} + \sqrt{dT\log(K)})$ lower bound when $d_t = d$ for all $t$. The matching upper bound was provided by \citet{zimmert2020optimal}, but nearly matching upper bounds also exist \citet{thune2019nonstochastic, bistritz2019online, gyorgy2020adapting,  van2021nonstochastic}. Conversely, (special cases of) combinatorial semi-bandits without delay have also received considerable attention \citep{gyorgy2007line, kale2010non, uchiya2010algorithms, audibert2014regret, combes2015combinatorial, lattimore2018toprank, zimmert2019beating}.

\paragraph{Adversarial Markov decision processes.} 
There is a rich literature on regret minimization in MDPs with non-delayed feedback \citep{even2009online,jaksch2010near,zimin2013online,rosenberg2019online,rosenberg2019bandit,rosenberg2021stochastic,jin2019learning,shani2020optimistic,luo2021policy}. Under delayed feedback, apart from the literature mentioned earlier, \cite{dai2022follow}
 recently presented a Follow-The-Perturbed-Leader approach that can also handle delayed feedback in adversarial MDPs. However, their regret bound is slightly weaker than \cite{jin2022near} mentioned earlier. Finally, a different line of work \citep{katsikopoulos2003markov,walsh2009learning} consider delays in observing the current state,
which is inherently different than our setting---for a thorough discussion on the differences between the models we refer the reader to \cite{lancewicki2020learning}.

\paragraph{Linear bandits.} Early work in the non-delayed linear bandit  setting suffered from suboptimal results in terms of $T$ \citep{mcmahan2004online, awerbuch2004adaptive, dani2006robbing}. \citet{abernethy2008competing} were the first to prove a regret bound with optimal scaling in $T$. Subsequent works by  \citep{bubeck2015entropic, hazan2016volumetric, ito2020tight, zimmert2022return} obtained the optimal $O(K\sqrt{T})$ regret bound.

\section{Preliminaries}

We denote by $\blhatt \in \reals^K$ the estimate of the loss $\bell_t$ in round $t$.  We will define a loss estimator for each application separately. We assume that delays $d_1,\ldots,d_T$ and losses $\bell_1,\ldots,\bell_T$ are both generated by an oblivious adversary. 
%We then define the expectation $\Emisst = \E_{\miss_t}[\cdot | \filobs_{t - 1}]$ as the expectation over all random elements in $\miss_t$, given all random elements that were observable in round $t - 1$. 
We focus on Follow The Regularized Leader (FTRL) which, in each round $t$, computes 
\begin{align}\label{eq:FTRL}
    \w_t = \argmin_{\w \in \domainw} \sum_{\tau \in \obs_t} \blhat_\tau^\top \w + R(\w),
\end{align}
where $\domainw\subseteq \reals^K$ is a compact closed convex set, $R$ is a twice-differentiable convex function such that $\nabla^2 R(\w) \succ 0\I$ for all $\w \in \domainw$, and $\obs_t = \{\tau: \tau + d_\tau < t\}$ is the set of indices of observed losses at the end of round $t-1$. Note that $\domainw$ and $\Aset$ do not necessarily coincide, as is the case of combinatorial semi-bandits for example. Similarly, $\action_t$ and $\w_t$ do not necessarily coincide. We will specify the relationship between $\action_t$ and $\w_t$ in each application. We define $\miss_t = [t-1] \setminus \obs_t$ to be the set of indices of losses that have not been observed at the end of round $t-1$ due to delay. As a simplifying assumption, we assume that $\dmax = \max_{t \in [T]} d_t \geq 1$ which is known to the learner. We also use the simplifying assumptions that $\sumT |\miss_t| = \totaldelay$ and $T$ are both known to the learner. These assumptions are without loss of generality, as we may employ the standard doubling trick to overcome the need to know these parameters \citep{bistritz2019online,lancewicki2020learning}, see also Appendix~\ref{sec:doubling}. We also make use of the following notations for FTRL on cumulative loss $\L$, which we denote by 
\begin{align*}
    \w(\L) = \argmin_{\w \in \domainw} \L^\top \w + R(\w).
\end{align*}
In the remainder of the paper we use the following cumulative losses: 
\begin{align*}
    \Lhat_t = \sum_{\tau \in \obs_t} \blhat_\tau, \qquad \bar{\L}_t^m = \Lhat_t + \sum_{\tau \in m_t} \bell_\tau, \qquad \Lhat_t^m = \Lhat_t + \sum_{\tau \in m_t} \blhat_\tau, \qquad \Lstar_t = \sum_{\tau \in [t]} \blhat_\tau%, \quad \Lhat_t^\star = \Lhat_t^m + \blhat_t.
\end{align*}
Note that $\Lstar_t = \Lhat_t + \sum_{\tau \in [t] \setminus \obs_t} \blhat_\tau$, $\w(\Lhat_t) = \w_t$ and that $\w(\Lhat_t^m)$ is equivalent to FTRL in the non-delayed setting.  

\paragraph{Additional notations.}
We define a filtration of all random events observed by the learner up to round $t$ as $\mathcal{F}_t =
\big\{\big(\tau, \action_\tau, \observe(\bell_\tau, \action_\tau)\big) \,:\, \tau + d_\tau < t\big\}$.
% $\{(\tau, a_\tau, \observe(\bell_\tau, A_\tau) : \tau: \tau + d_\tau < t\}$.
For a twice-differentiable function $\phi$ such that $\nabla^2 \phi(\w) \succ 0\I$ for all $\w \in \domainw$ we will denote by $\|\L\|_{\phi, \w} = \sqrt{\L^\top \big(\nabla^2 \phi (\w)\big)^{-1} \L}$ and by $\|\L\|_{\phi, \w}^\star = \sqrt{\L^\top \nabla^2 \phi (\w) \L}$. The Dikin ellipsoid with radius $r$ around $\w$ induced by $\phi$ is defined as $\dikin_\phi(\w, r) = \{\x \in \domainw: \|\x - \w\|_{\phi,\w}^\star \leq r\}$. The notations $\Tilde{O}(\cdot)$ and $\lesssim$ hide poly-logarithmic factors.

\section{Analysis}\label{sec:analysis}
We build on the analysis of \citet{flaspohler2021online} for delayed feedback in the full-information setting, where they observe that delayed feedback can be interpreted as poor hints in the sense of optimistic online learning \citep{rakhlin2013online}. Taking this idea one step further, we analyze what would happen had the algorithm received slightly different hints, and subsequently bound the difference between different instances of FTRL. 

We assume that our loss estimates satisfy $\E[\blhat_t|\mathcal{F}_t] = \bell_t + \b_t$, where $\b_t$ is the estimator's bias. 
Our analysis relies on the following decomposition of the instantaneous regret
\begin{align}\label{eq:decomp}
    &\sumT \E\big[(\w_t - \u)^\top \bell_t \big] = \underbrace{\sumT -\E\big[(\w(\Lhat_t^m) - \u)^\top \bias_t\big]}_{\textnormal{bias}} + \underbrace{\sumT \E\big[(\w(\Lstar_t) - \u)^\top \blhat_t\big]}_{\textnormal{cheating regret}}\\ \nonumber
    & + \sumT\Big(\E\big[\underbrace{(\w_t - \w(\bar{\L}_t^m))^\top \bell_t}_{\textnormal{\hinta}}\big] + \E\big[\underbrace{(\w(\bar{\L}_t^m) - \w(\Lhat_t^m))^\top \bell_t}_{\textnormal{\hintb}}\big] + \E\big[\underbrace{(\w(\Lhat_t^m) - \w(\Lstar_t))^\top \blhat_t}_{\textnormal{\hintc}}\big]\Big).
\end{align}
Suppose that $\b_t = \0$, i.e., $\blhat_t$ is an unbiased estimator of the loss. This implies that the bias term of the decomposition is $0$. The cheating regret can be found in different forms in online learning---see, for example, the proof of \citep[Lemma~2.3]{shalev2012online} or \citep[Equation~4]{gyorgy2020adapting}---and can be bounded using the standard be-the-leader lemma (Lemma~\ref{lem:BTL} in Appendix~\ref{appendix:proofs-for-general-analysis}). Now, let us focus on the second line of Equation~\eqref{eq:decomp}. Once simplified, the second line becomes 
$%\begin{align*}
    \sumT \E\big[(\w_t - \w(\Lhat_t^\star))^\top \blhat_t\big],
$%\end{align*}
which can be recognised as the drift term of \citet[Equation~(4)]{gyorgy2020adapting}. Normally, we would like to use standard tools from linear bandits or online convex optimization to bound such terms, such as for example (a variation of) Lemma~\ref{lem:FTRLnorm} below, the proof of which can be found in Appendix~\ref{appendix:proofs-for-general-analysis}. 
\begin{restatable}{relemma}{lemmaFTRLnorm}
\label{lem:FTRLnorm}
Suppose that $\nabla^2 R( \w') \succeq \frac{1}{4} \nabla^2 R( \w) $ for all $\w \in \domainw$, $\w' \in \dikin_R(\w, \half)$, and for some twice-differentiable convex $R$. Let $\x \in \domainw$ and $\L,\L'\in\reals^K$ such that $\w(\L'), \w(\L) \in \dikin_R(\x, \half)$, then 
$
    \|\w(\L) - \w(\L')\|_{R, \x}^\star \leq 8\|\L' - \L\|_{R, \x}
$.
\end{restatable}
%
% Suppose we are able to apply Lemma~\ref{lem:FTRLnorm} to further bound $(\w_t - \w(\Lhat_t^\star))^\top \blhat_t \leq \|\w_t - \w(\Lhat_t^\star)\|_{R, \w_t}^\star\|\blhat_t\|_{R, \w_t}$, which is H\"olders inequality. This would lead to the term 
We can bound the drift term $(\w_t - \w(\Lhat_t^\star))^\top \blhat_t$ by $\|\w_t - \w(\Lhat_t^\star)\|_{R, \w_t}^\star\|\blhat_t\|_{R, \w_t}$ using H\"older inequality, and then apply Lemma~\ref{lem:FTRLnorm} to further bound the right-hand side. This would lead to the problematic term 
\begin{align*}
    \sumT \E\left[\bigg\|\sum_{\tau \in [t] \setminus \obs_t} \blhat_\tau\bigg\|_{R, \w_t}\|\blhat_t\|_{R, \w_t}\right]
    \le \sumT \E\left[(1 + |m_t|)\max_{\tau \in [t] \setminus \obs_t}\|\blhat_\tau\|_{R, \w_t}^2\right]~.
\end{align*}
To see where the problem is, suppose we are in the multi-armed bandit setting, $R$ is the negative entropy scaled by $\frac{1}{\eta}$, and $\blhat_t$ is the standard importance-weighted estimator. Then, the upper bound above is $O(\eta (1 + |m_t|) K)$, where $K$ is the number of arms. Summing over $T$ rounds, using a $\frac{\log(K)}{\eta}$ bound on the cheating regret, and tuning $\eta$, we see that this analysis delivers a regret of order $O\big(\sqrt{K(T+D)\log (K)}\big)$ where $D = \sum_t |m_t|$. In the case of constant delay $d_t = d$, the bound becomes $O\big(\sqrt{KdT \log(K)}\big)$ which is known to be suboptimal, as the minimax regret in this case is of order $\max\big\{\sqrt{dT\log(K)}, \sqrt{KT}\big\}$ \citep{cesa2019delay, zimmert2020optimal}. 

The intuition behind the suboptimality of the above analysis is that the cost of bandit feedback and the cost of delayed feedback are not separated. Indeed, the analysis of most lower bounds is split in two cases: a lower bound for bandit feedback without delay and a lower bound for delayed \emph{full-information} feedback, see for example \cite{cesa2019delay}. Separating the impact of delayed feedback and bandit feedback is precisely why we bound the terms \hinta, \hintb, and \hintc\ of Equation~\eqref{eq:decomp} separately, which leads to Lemma~\ref{lem:main} below, whose proof can be found in Appendix~\ref{appendix:proofs-for-general-analysis}.
\begin{restatable}{relemma}{lemmamain}
\label{lem:main}
    Let $R$ be convex and twice-differentiable such that $4 \nabla^2 R( \w) \succeq \nabla^2 R( \w') \succeq \frac{1}{4} \nabla^2 R( \w)$ for all $\w \in \domainw$ and $\w' \in \dikin_R(\w, \half)$. Let $\|\bell_t\|_{R, \w} \leq \fbound \leq \frac{1}{16\dmax}$ for all $t$ and $\w \in \domainw$, and $\E\big[\|\blhat_t\|_{R, \w_t}^2\big] \leq \bbound^2$ for all $t$, where $\w_t$ is given by~\eqref{eq:FTRL}. %, that $\sum_{\tau \in [t] \setminus \obs_t} \|\blhat_\tau\|_{R, \w(\Lhat_t^m)} \leq \frac{1}{8}$, that $\sum_{\tau \in m_t \setminus \obs_t} \|\blhat_\tau\|_{R, \w(\Lhat_t^m)} \leq \frac{1}{8}$, 
    Suppose also that $\|\blhat_t\|_{R, \w_t} \leq \frac{1}{64(1 + \dmax)}$  for all $t$ with probability 1. Then for all $\u \in \domainw$
    \begin{align*}
        \E \left[ \sumT (\w_t - \u)^\top \bell_t \right] & \leq R(\u) - R(\w_1) + 8\bbound^2T -\sumT \E\big[(\w(\Lhat_t^m) - \u)^\top \bias_t\big]\\        
        & \quad + \sumT \left(8\fbound^2|\miss_t| + 8\fbound\E\left[\bigg\|\sum_{\tau \in \miss_t}(\bell_\tau - \blhat_\tau)\bigg\|_{R, \w_t}\right]\right)~.
    \end{align*}
\end{restatable}

To interpret Lemma~\ref{lem:main}, consider the multi-armed bandit setting with the standard importance-weighted estimator and regularizer $R(\w) = \sum_{i = 1}^K \frac{1}{\eta}\w(i)\log(\w(i)) - \frac{1}{\gamma}\log(\w(i))$. The purpose of the log barrier term in the regularizer is to ensure stability of the iterates, as required by the assumptions of the lemma. In this case, if $\|\bell_t\|_\infty \leq 1$, then $\fbound$ is $O(\sqrt{\eta})$. The quantity $\bbound^2$ is a bound on the expectation of the squared local norm of the loss estimate, which is $O(\eta K)$. Thus, we have that the regret is of order 
\begin{align}\label{eq:delayedMABbandit}
    \frac{1}{\eta}\log(K) + \dmax^2 K\ln(T) + \eta (K T + D) + \sumT \fbound\E\left[\bigg\|\sum_{\tau \in m_t}(\bell_\tau - \blhat_\tau)\bigg\|_{R, \w_t}\right]~,
\end{align}
where we used that $\sumT|m_t| = \totaldelay$. The $\dmax^2 K\ln(T)$ term in the above equation comes from the log-barrier part of $R$, which when tuned properly is able to ensure that the iterates of FTRL are close to each other.  
So far, it seems that we did not manage to separate the cost of delay and bandit feedback because of the final summation in~\eqref{eq:delayedMABbandit}. However, observe that due to the delay, for $\tau, \tau' \in \miss_t$, $\blhat_\tau$ and $\blhat_{\tau'}$ are independent random variables and $\bell_\tau$ and $\bell_{\tau'}$ are their means.
% Recall that for independent random variables $x$ and $x'$ we have that $\E[(x + x' - \E[x + x'])^2] = \E[(x - \E[x])^2] + \E[(x' - \E[x'])^2]$.
Recall that the variance of the sum of independent random variables equals to the sum of their variances.
Thus, by applying Jensen's inequality to the square root, and using that $\nabla^2 R(\w) \succeq \textnormal{diag}\big(\eta \w \big)^{-1}$, we can see that 
%
%\begin{align*}
%    & \fbound \E\left[\bigg\|\sum_{\tau \in m_t}(\bell_\tau - \blhat_\tau)\bigg\|_{R, \w_t}\right] \leq \fbound\sqrt{\E\left[\sum_{i = 1}^K \eta \w_t(i)\left(\sum_{\tau \in m_t}(\bell_\tau(i) - \blhat_\tau(i))\right)^2\right]}\\
%    & =  \fbound\sqrt{\E\left[\sum_{i = 1}^K \eta \w_t(i) \sum_{\tau \in m_t}\big(\bell_\tau(i) - \blhat_\tau(i)\big)^2\right]} \leq 2\fbound\sqrt{\E\left[\sum_{i = 1}^K \sum_{\tau \in m_t}\eta \w_{\tau}(i)\big(\bell_\tau(i) - \blhat_\tau(i)\big)^2\right]} \\
%    & \leq \sqrt{\fbound^2  |m_t| \eta K}~,
%\end{align*}
\begin{align*}
    & \fbound \E\left[\bigg\|\sum_{\tau \in m_t}(\bell_\tau - \blhat_\tau)\bigg\|_{R, \w_t}\right]
    \leq 2 \fbound \E\left[\bigg\|\sum_{\tau \in m_t}(\bell_\tau - \blhat_\tau)\bigg\|_{R, \w_\tau}\right] \\
    & \leq 2 \fbound \sqrt{\E\left[\sum_{i = 1}^K \eta \w_\tau(i)\left(\sum_{\tau \in m_t}(\bell_\tau(i) - \blhat_\tau(i))\right)^2\right]}
    = 2\fbound\sqrt{\E\left[ \sum_{\tau \in m_t} \sum_{i = 1}^K \eta \w_{\tau}(i)\big(\bell_\tau(i) - \blhat_\tau(i)\big)^2\right]} \\
    & \leq \sqrt{\fbound^2 |m_t| \eta  K}~,
\end{align*}
where the first inequality is due to Lemma~\ref{lem:induction} in Appendix~\ref{appendix:proofs-for-general-analysis}, a new result that proves the multi-round stability of FTRL iterates under certain conditions, which can be applied for sufficiently small $\gamma$.  Recalling that $\fbound$ is $O(\sqrt{\eta})$ and using $\sqrt{\eta |m_t| \eta K} \leq \half (\eta |m_t| + \eta K)$ we can see that~\eqref{eq:delayedMABbandit} is in fact of order
$
     \log(K)/\eta + \dmax^2 K\ln(T) + \eta (K T + D)
$,
which gives a $O\big(\sqrt{(KT + D) \log(K)} + \dmax^2 K\ln(T)\big)$ bound for an appropriately tuned $\eta$. 

To conclude, as long as loss estimates $\blhat_{\tau}$ and $\blhat_{\tau'}$ are independent for $\tau, \tau' \in \miss_t$, Lemma~\ref{lem:main} implies that we have effectively split the cost of delayed feedback and bandit feedback. We formalize the above in Corollary~\ref{cor:main}, whose proof can be found in Appendix~\ref{appendix:proofs-for-general-analysis}. 
%To see that Lemma~\ref{lem:main} is useful, let us assume that, conditioned on $\mathcal{F}_t$, $\blhat_\tau$ and $\blhat_{\tau'}$ are independent. This independence is standard in online learning with delayed feedback, as the distribution of the estimator in round $\tau$ is independent of $\blhat_{\tau'}$ as the loss estimate of round $\tau'$ is not used to choose the action of round $\tau$ due to the delay and vice versa. This implies Corollary~\ref{cor:main}.
%
\begin{restatable}{recorollary}{corollarymain}
\label{cor:main}
    Under the same assumptions as in Lemma~\ref{lem:main}, suppose that $\E[\blhat_\tau|\mathcal{F}_t] = \bell_\tau$ and that
$
\E\Big[(\blhat_\tau - \bell_\tau)^\top\big(\nabla^2 R(\w(\Lhat_t))\big)^{-1} (\blhat_{\tau'} - \bell_{\tau'})\,\Big|\, \mathcal{F}_t\Big] = 0
$
for all $t \in [T]$ and all $ \tau, \tau' \in \miss_t$ where $\tau' \neq \tau$. Then for all $\u \in \domainw$
    \begin{align*}
        \E \left[ \sumT (\w_t - \u)^\top \bell_t \right] & \leq R(\u) - R(\w_1) + 16\bbound^2 T + 16\fbound^2 D~. %\sumT |m_t| .
    \end{align*}
\end{restatable}

\section{Combinatorial Bandits}
\label{sec: CMAB}   
In this section, we demonstrate how to apply our generic FTRL approach to combinatorial bandits (CMAB) with delayed feedback. This yields the first algorithm to achieves optimal regret in that setting. We start with the description of the model as an instance of our general setting.

%The multi-armed bandit problem has been extend to various special cases 
% We introduce the combinatorial bandit framework with semi-bandit feedback. In each round $t = 1,\ldots, T$ the environment picks $\bell_t \in [-1, 1]^\actiondim$ and the learner picks an action $\action_t \in \actionset \subseteq \{0, 1\}^\actiondim$. The learner then suffers loss $\action_t^\top \bell_t$ and subsequently observes $\action_\tau \odot \bell_\tau$ where $\tau + \doft{\tau} = t$.
Delayed combinatorial bandits with semi-bandit feedback is an instance of the online learning framework where $\bell_t \in [-1, 1]^\actiondim$, $\actionset \subseteq \{0, 1\}^\actiondim$, $\domainw = \Conv(\actionset)$ (the convex hull of $\actionset$), and the feedback function is $\observe(\bell_{\tau}, \action_{\tau}) = \action_\tau \odot \bell_\tau$ where $\odot$ is the Hadamard (elementwise) vector product.
% \begin{enumerate}[nosep]
%     \item The environment picks $\bell_t \in [-1, 1]^\actiondim$, learner picks an action $\action_t \in \actionset \subset \{0, 1\}^\actiondim$.
%     \item The learner suffers $\action_t^\top \bell_t$.
%     \item The learner observers $\action_\tau \odot \bell_\tau$, where $\tau + \doft{\tau} \leq t$ for an adversarial delay $\doft{\tau}$.
% \end{enumerate}

 % and assume that $\|\action\|_1 \leq \combm$ for all $\action \in \actionset$.
 We define the pseudo-regret in this setting as 
\begin{align*}
    \regret_T = \E \left[ \sumT (\action_t - \action^*)^\top \bell_t \right],
\end{align*}
with $\action^* = \argmin_{\action \in \actionset} \sumT \action^{\top} \bell_t$.
\begin{algorithm}[t]
\label{alg:CMAB}
\caption{Delayed FTRL for combinatorial bandits}
\begin{algorithmic}
    \STATE \textbf{Input:} $\gamma \in (0, 1)$, $\eta$.
\FOR{$t \in [T]$}
    \STATE Compute $\w_t = \argmin_{\w \in \domainw} \sum_{\tau \in \obs_t} \blhat_\tau^\top \w + R(\w)$ with $R$ as in equation~\eqref{eq:cmabregularizer}.\;
    \STATE Find probability distribution $\p_t$ such that $\E_{\action \sim \p_t}[\action] = \w_t$.\;
    \STATE Draw and play $\action_t \sim \p_t$.\;
    \STATE Observe loss $\action_\tau \odot l_\tau$ and compute $\blhat_\tau(i) = \frac{\action_\tau(i) \bell_\tau(i)}{\w_\tau(i)}$ for all $\tau \in \obs_t$.\;
    \ENDFOR
\end{algorithmic}
\end{algorithm}
The algorithm we use in this setting (Algorithm~\ref{alg:CMAB}) is inspired by the algorithm of \citet{audibert2014regret}. In any given round $t$, Algorithm~\ref{alg:CMAB} first computes $\w_t$, the solution of the FTRL optimization problem (Eq.~\eqref{eq:FTRL}) over the convex hull of the action set. Then, it constructs a probability distribution $\p_t$ over $\Aset$ such that $\E_{\a \sim \p_t}[\a] = \w_t$. The estimator of loss is given by $\blhat_t(i) = \frac{\action_t(i) \bell_t(i)}{\w_t(i)}$, which is clearly unbiased. 
We use the regularizer
\begin{align}\label{eq:cmabregularizer}
    R(\w) = \sumdim \Big(\frac{1}{\eta}\w(i)\log(\w(i)) - \frac{1}{\gamma}\log(\w(i))\Big),
\end{align}
where $\eta > 0$ and $\gamma > 0$ are hyperparameters. We now state the main results of this section.
\begin{restatable}{retheorem}{thCMAB} 
\label{th: CMAB regret}
Let $\max_{\action\in \actionset}\|\action\|_1 \leq \combm$. 
Running Delayed FTRL for combinatorial bandits (Algorithm~\ref{alg:CMAB}) with $\gamma = \frac{1}{64^2 \combm (1 + \dmax)^2}$ and $\eta = \min \l\{ \frac{1}{16^2 \combm \dmax^2}, \sqrt{\frac{\combm \l( 1 + \ln \l( \frac{\actiondim}{\combm} \r)\r) }{ 16 (\actiondim T + \combm \totaldelay) } } \r\}$
guarantees, 
$$
       \regret_T \leq 12 \sqrt{\combm (\actiondim T + \combm \totaldelay) \ln \bigg( \frac{\actiondim}{\combm} \bigg)} + 64^2 \combm (1 + \dmax)^2 \actiondim \ln (T) + 2^9 \combm^2 \dmax^2 \ln \bigg( \frac{\actiondim}{\combm} \bigg).
$$
\end{restatable}
We sketch the proof of Theorem~\ref{th: CMAB regret} (see Appendix~\ref{app: CMAB} for a full proof) which follows from an application of Corollary~\ref{cor:main}. To apply this corollary we first need to verify its assumptions. The assumptions on $\nabla^2 R(\w)$ are implied by Lemma~\ref{lem:dikin implies factor 2} in Appendix~\ref{app:auxiliary}. Next, we set $\fbound = \sqrt{\eta \combm}$ and verify that $\frac{1}{16\dmax} \geq \fbound \geq \|\bell_t\|_{R, \w}$. Indeed, $\w\in\Conv(\actionset)$ implies $\sumdim \w(i) \leq \combm$. This, together with $|\bell(i)| \leq 1$ and $\eta + \gamma \w(i) \geq \gamma \w(i)$, implies
\begin{align*}
    \|\bell_t\|_{R, \w} = \sqrt{\bell_t^\top \big(\nabla^2 R (\w)\big)^{-1} \bell_t} = \sqrt{ \sumdim \bell_t(i)^2 \frac{\eta \gamma \w^2(i)}{\eta + \gamma \w(i)}} \leq \sqrt{\eta \combm} = \fbound~.
\end{align*}
Our choice of $\eta$ % $\eta \leq \frac{1}{16^2 \combm \dmax^2}$
then implies $\fbound \leq \frac{1}{16\dmax}$. Similarly, by setting $\bbound^2 = \eta \actiondim$ we fulfill the condition $\E \big[ \|\blhat_t\|_{R, \w_t}^2 | \filobs_{t} \big] \le \bbound^2$. Finally, the condition on $\gamma$ allows us to verify the assumption $\|\blhat_t\|_{R, \w_t} \leq \frac{1}{64 (1 + \dmax)}$.
% at the cost of introducing the condition that $\gamma \leq \frac{1}{64^2 \combm (1 + \dmax)^2}$.
Since all the assumptions are verified, we can now apply Corollary~\ref{cor:main} and finish the proof of Theorem~\ref{th: CMAB regret}. 
The Algorithm is computationally efficient for a range of action-sets including $m$-sets and spanning trees. In general the algorithm is efficient whenever OSMD of \citet{audibert2014regret} is efficient and we refer to that paper for more details.
%
    % After being able to apply Corollary~\ref{cor:main}, to find that
    % \begin{align*}
    %     \E \left[ \sumT (\action_t - \u)^\top \bell_t \right] &\leq R(\u) - R(\w_1) + 12 \eta \actiondim T + 12 \eta \combm \totaldelay,
    % \end{align*}
    % for all $\u \in \domainw$, the only challenge that is left is to control the regulariser $R$. Since $R(\u)$ is unbounded in general, we cannot simply compare with $\action^*$ which might lie on the boundary on $\domainw$. Instead we define an $\u = (1 - \combslider)\action^* + \combslider \frac{\1}{\actiondim}$ for an $\combslider \in [0, 1]$, that will be our comperator. $\combslider$ now acts as a trade-off between an upper bound on the regulariser and an additional bias-like term that stems from comparing $\a^*$ to $\u$ in terms of regret. This allows us to find
    % \begin{align*}
    %     R(\u) - R(\w_1) \leq \frac{ \combm \ln \l( \frac{\actiondim}{\combm} \r)}{\eta} + \frac{\actiondim \ln\l( \frac{\actiondim}{\combslider}\r)}{\gamma}.
    % \end{align*}    
    % Combining the above with $\regret_T \leq \E \left[ \sumT (\action_t - \u)^\top \bell_t \right] + 2 \combslider \combm T$, tuning for an optimal $\combslider$, and finding appropriate $\gamma$ and $\eta$ concludes the proof already.

We now state a lower bound for the delayed combinatorial semi-bandit setting. This implies that, ignoring terms that are logarithmic in $T$, the result of Theorem~\ref{th: CMAB regret} is optimal. The proof of our lower bound follows from standard arguments in the delayed bandit feedback literature and can be found in Appendix~\ref{app: CMAB}.
\begin{restatable}{retheorem}{thLBcombinatorial}\label{th:lowerboundcombinatorial}
    Suppose that $d_t = d$ for all $t$ and that $B \leq K/2$ . Then there exists a sequence of losses such that any algorithm in the delayed combinatorial semi-bandit setting satisfies
    \begin{align*}
        \E\left[\sumT (\action_t - \action^\star)^\top \bell_t\right] = \Omega\left(\max\Big\{\sqrt{BKT}, B\sqrt{d T} \Big\}\right)\enspace.
    \end{align*}
\end{restatable}

\section{Adversarial Markov Decision Processes}
\label{sec:MDPs}

In this section, we apply our FTRL approach to adversarial Markov Decision Processes (MDPs) where the transition function is known to the learner in advance. We show that it yields the first algorithm that handles delay optimally in this setting.
We start with a presentation of the model and regret minimization framework.

A finite-horizon episodic adversarial MDP is defined by $\calM = (\mdpS , \mdpA , H , p, \{ \mdpell_t \}_{t=1}^T,\sinit)$, where
$\mdpS$ and $\mdpA$ are finite state and action spaces of sizes $S$ and $A$, respectively, $H$ is the horizon, $T$ is the number of episodes, and $\sinit \in \mdpS$ is the initial state. 
$p = \{ p_h: \mdpS \times \mdpA \to \Delta_{\mdpS} \}_{h=1}^H$ is the \textit{transition function} such that $p_h(s' \mid s,a)$ is the probability of moving to $s'$ when taking action $a$ in state $s$ at time $h$. 
$\{ \mdpell_{t,h} : \mdpS \times \mdpA \to [0,1] \}_{t=1,h=1}^{T\ \ \  ,H}$ are \textit{cost functions} chosen by an \textit{oblivious adversary}, where $\mdpell_{t,h}(s,a)$ is the cost of taking action $a$ in state $s$ at time $h$ of episode $t$. For the ease of presentation, we slightly abuse notation and treat each set of loss functions $\{ \mdpell_{t,h} : \mdpS \times \mdpA \to [0,1] \}_{h=1}^{H}$ simply as a vector $\mdpell_t \in [0,1]^{HSA}$.

The learner interacts with the environment over $T$ episodes.
At the beginning of episode $t$, it picks a policy $\pi_t = \{ \pi_{t,h}: \mdpS \to \Delta_\mdpA \}_{h=1}^H$, and starts in the initial state $s_{t,1} = \sinit$.
At each time $h\in [H]$, it observes the current state $s_{t,h} \in \mdpS$, draws an action from the policy $a_{t,h} \sim \pi_{t,h}(\cdot \mid s_{t,h})$ and transitions to the next state $s_{t,h+1} \sim p_h(\cdot \mid s_{t,h},a_{t,h})$. 
The feedback of episode $t$ contains the cost function over the agent's trajectory $\{ \mdpell_{t,h}(s_{t,h},a_{t,h}) \}_{h=1}^H$, i.e., bandit feedback, 
and is observed only at the end of episode $t+d_t$. The learner's goal is to minimize the value of its policies, where $V^{\pi}_{t,h}(s) = \bbE [ \sum_{h'=h}^{H} \mdpell_{t,h'}(s_{h'},a_{h'}) \mid s_{h}=s,\pi,p ]$ is the value function of policy $\pi$ with respect to the cost $\mdpell_t$.
 The performance is measured by the \textit{regret}, defined as the difference between the cumulative expected cost of the learner and the best fixed policy in hindsight:
$
    \regret_T
    =
    \sumT V^{\pi_t}_{t,1}(\sinit) - \min_{\pi \in \Pi} \sumT V^{\pi}_{t,1}(\sinit).
$

In order present the adversarial MDP model as an instance of the general online learning framework we use the notion of \textit{occupancy measures}. Given a policy $\pi$ and a transition function $p'$, the \textit{occupancy measure} $\w^{\pi,p'} \in [0,1]^{HS^2A}$ is a vector, where $\w^{\pi,p'}_h(s,a,s')$ is the probability to visit state $s$ at time $h$, take action $a$ and transition to state $s'$. 
We also denote $\w^{\pi,p'}_h(s,a) = \sum_{s'} \w^{\pi,p'}_h(s,a,s')$ and $\w^{\pi,p'}_h(s) = \sum_{a} \w^{\pi,p'}_h(s,a)$. 
By \cite{rosenberg2019online}, the occupancy measure encodes the policy and the transition function through the relations
$
    \pi_{h}(a \mid s) 
    = \nicefrac{\w^{\pi,p'}_{h}(s,a)}{\w^{\pi,p'}_h(s)}
    ; \,\, 
    p'_h(s'\mid s,a) 
    = \nicefrac{\w^{\pi,p'}_{h}(s,a,s')}{\w^{\pi,p'}_{h}(s,a)}.
$
The set of all occupancy measures with respect to an MDP $\calM$ is denoted by $\Delta(\calM)$, and the set of all policies by $\Pi = \left\{ \{ \pi_h:\mdpS \to \Delta_\mdpA \}_{h=1}^H \right\}$.
Importantly, the value of a policy from the initial state (i.e., the expected loss in an episode) can be written as the dot product between its occupancy measure and the cost function, i.e., $ \langle \w^{\pi,p'} , \mdpell \rangle = \sum_{h,s,a} \w^{\pi,p'}_h(s,a) \mdpell_h(s,a)$.
Thus, the regret becomes $\regret_T = \sumT \langle \w^{\pi_t,p} , \mdpell_t \rangle - \min_{\w \in \Delta(\calM)} \sumT \langle \w , \mdpell_t \rangle$.
Whenever $p'$ is omitted from the notation $\w^{\pi,p'}$, it is understood to be the true transition function $p$. 

\begin{algorithm}[t]
    \label{alg:MDPs}
    \caption{Delayed FTRL for adversarial MDPs}

    \begin{algorithmic}
    
    \FOR{$t = 1,...,T$}

        \STATE Compute $\w_{t} = \argmin_{\w \in \domainw} \hat{\L}_t^\top \w + R(\w)$, and 
        policy $\pi_{t,h}(a \mid s) = \frac{\w_{t,h}(s,a)}{\w_{t,h}(s)} \,\,\, \forall (s,a,h)$.

        \STATE Play episode $t$ with policy $\pi_t$, and observe feedback $\{ \mdpell_{t,h}(s_{\tau,h},a_{\tau,h}) \}_{h=1}^H$ for all $\tau + d_{\tau} = t$.

        \STATE Compute upper occupancy bound $\u_{\tau,h}(s,a) = \max_{\hat{p} \in \calP} \w^{\pi_\tau,\hat{p}}_h(s,a)$.

        \STATE Compute loss estimator $\hat\mdpell_{\tau,h}(s,a) = \frac{\bbI\{ s_{\tau,h} = s,a_{\tau,h} = a \} \mdpell_{\tau,h}(s,a)}{\u_{\tau,h}(s,a)}$ and update $\hat L_t$.

    \ENDFOR
\end{algorithmic}
\end{algorithm}

With that in hand, adversarial MDPs is an instance of the online learning framework where $\mdpell_t \in [0,1]^{HSA}$, $\Aset$ as the set of occupancy measures $\Delta(\calM)$ and the feedback  $\observe(\w^{\pi_\tau},\mdpell_\tau)$ is the loss over the trajectory $\{\mdpell_{\tau,h}(s_{\tau,h},a_{\tau,h})\}_{h=1}^H$. $\domainw$ is a (slightly modified) set of occupancy measures which we will define later.

Next, we present our FTRL algorithm (Algorithm \ref{alg:MDPs}), based on the general framework presented in Section~\ref{sec:analysis}. To satisfy the stability conditions required for Lemma \ref{lem:main}, we employ a hybrid regularization of negative entropy and log-barrier similar to the combinatorial bandit case: $R(\w) 
        = 
        \frac{1}{\eta} \sum_{h,s,a,s'} \w_h(s,a,s') \log \w_h(s,a,s') - \frac{1}{\gamma} \sum_{h,s,a,s'} \log \w_h(s,a,s')$.
    The main difference is that some of the elements of the occupancy measures may be $0$ regardless of the chosen policy (e.g., if $p_h(s'\mid s,a) = 0 $, then $ \w_h^\pi (s,a,s') = 0$), in which case the regularization is not well-defined. To avoid that, we first augment the set of occupancy measures as follows:   $\Delta(\calP)
        = 
        \{\w^{\pi,\hat p}: \pi \in \Pi, \hat{p}\in \calP \}$ where 
        $\calP 
            = 
        \left\{ \{ \hat{p}_h \}_{h=1}^H: \forall h, \,\,\|\hat{p}_h - p_h\|_{\infty} \leq \frac{1}{THSA} \right\}.
        $ 
        Then, we intersect it with $\Omega 
        = 
        \left\{ \w\in [0,1]^{HS^2 A} : \forall (h,s,a,s').\,\, \w_h(s,a,s') \geq \frac{1}{T^3 H^2 S^4 A^2}  \right\}$. 
        % $\domainw = \Delta(\calP)\cap \Omega$ where .
  This construction allows us to establish the following properties:
\begin{restatable}{relemma}{lemmaMDPestimatorbias}
\label{lemma:close-w-in-W-mdp-known-dynamics}
\label{lem:MDPs known p trancation error}
    Let $\domainw = \Delta(\calP)\cap \Omega$. It holds that $\domainw$ is non-empty and,
   \begin{enumerate}
       \item 
    For any $\w \in \Delta(\calM)$, there exists $\tilde{\w} \in \domainw$ such that $\Vert \w -  \tilde{\w} \Vert_1 \le \frac{2H}{T}$.
    \item Given $\w \in \domainw$, let $\pi$ be defined by $\pi_{h}(a \mid s) = \frac{\w_{h}(s,a)}{\w_{h}(s)}$ and $\u_{h}(s,a) = \max_{\hat{p} \in \calP} \w^{\pi,\hat{p}}_h(s,a)$.
    Then, $\| \w^{\pi} - \w \|_1 \leq \frac{2H}{T}$ and $\| \u - \w \|_1      \leq \frac{4 H^2 S}{T}$.
    \end{enumerate}
\end{restatable}
% That is, for any occupancy measure with respect to $p$, there exist an occupancy measure in the domain which is close up to a small error; and for any occupancy measure in the domain, the true occupancy measure of the policy 
  With that in hand, the regularization $R$ is well-defined on the the domain $\domainw$ and bounded by $\Tilde{O}(\frac{1}{\eta} + \frac{H S^2 A}{\gamma})$. Moreover, we are guaranteed that given the iterate $\w_t$ and the corresponding policy $\pi_{t,h}(a\mid s) = \w_{t,h}(s,a)/\w_{t,h}(s)$, the true occupancy measure $\w^{\pi_t}$ is close to $\w_t$ up to a small error. 
  Next, to keep the local norm of the estimator small (which affects the guarantee of Lemma \ref{lem:main}), we introduce a slightly biased importance-sampling estimator (inspired by \citet{jin2019learning}) defined by $\hat\mdpell_{t,h}(s,a) = \frac{\bbI\{ s_{t,h} = s,a_{t,h} = a \} \mdpell_{t,h}(s,a)}{\u_{t,h}(s,a)}$ where $\u_{t,h}(s,a) = \max_{\hat{p} \in \calP} \w^{\pi_t,\hat{p}}_h(s,a)$. Recall that the local norm is evaluated at $\w_t$ and note that the expectation of the indicator $\bbI\{ s_{t,h} = s,a_{t,h} = a \}$ is  $w^{\pi_t}_h(s,a)$. Thus, the fact that $\u_{t,h}(s,a)$ upper bounds both $\w_{t,h}(s,a)$ and $\w^{\pi_t}_{h}(s,a)$ would allow us to keep local norm small. In addition, using the second property in Lemma \ref{lem:MDPs known p trancation error}, we can also show that the estimator's bias is only of order $1/T$ (ignoring $S,H$ factors).
Finally, note that $\domainw$ is a convex set defined by linear constrains, and thus the optimization can be solved efficiently \citep{rosenberg2019online,lee2020bias}. In addition, $\u_t$ can be computed efficiently as well using dynamic programming \citep{jin2019learning}.

We finish this section with the regret bound of Algorithm~\ref{alg:MDPs}.
The proof relies on the following regret decomposition 
    \begin{align*}
        \regret_T = 
        \sum_{t=1}^T\langle \w^{\pi_{t}} - \w^{\star}, \mdpell_{t} \rangle
        = \underbrace{
        \sum_{t=1}^T \langle\w^{\pi_{t}} - \w_{t}, \mdpell_{t}\rangle
                    }_{\textsc{Error}} 
            +  \underbrace{ \sum_{t=1}^T
            \langle\w_{t} - \tilde{\w}^{\star}, \mdpell_{t}\rangle
                        }_{\textsc{Reg}}
            +  \underbrace{ \sum_{t=1}^T
            \langle \tilde{\w}^{\star} - \w^{\star}, \mdpell_{t}\rangle
                        }_{\textsc{Shift-penalty}},
    \end{align*}
    where, by Lemma~\ref{lem:MDPs known p trancation error}, \textsc{Error} is bounded by $2H$ and $\tilde{w}^{\star} \in \domainw$ exists such that \textsc{Shift-penalty} is bounded by $2H$. Finally, \textsc{Reg} is bounded by utilizing Lemma~\ref{lem:main}. Due to lack of space we defer the full details of the proof to Appendix \ref{appendix: MDPs}.

\begin{restatable}{retheorem}{thmMDPmain} 
Running Delayed FTRL for adversarial MDPs (Algorithm~\ref{alg:MDPs}) with $\gamma = \frac{1}{4096 H (1+d_{max})^2}$ and $\eta = \min \l\{\frac{1}{256 H (1+d_{max})^2}, 
\frac{1}{\sqrt{(S A T + D)\log(H S A T)}} \r\}$ guarantees
\[
    \bbE[\regret_T]
        \leq 10 H \sqrt{S A T \log(H S A T)} + 10 H\sqrt{D \log(H S A T)} + 7\cdot 10^5 H^2 S^2 A (1+d_{max})^2.
\]
\end{restatable}
Remarkably, the above regret bound matches the lower bound of \cite{lancewicki2020learning} (up to poly-logarithmic factors), making
it the first optimal regret for adversarial MDPs with delayed bandit feedback. 

\section{Linear Bandits}
\label{sec:linear-bandits}

\begin{algorithm}[t]
    \label{alg:linban}
    \caption{Delayed FTRL for linear bandits}

    \begin{algorithmic}    
    \FOR{$t = 1,...,T$}

        \STATE Compute $\w_{t} = \argmin_{\w \in \domainw} \hat{\L}_t^\top \w + R(\w)$, where $R(\w) = \frac{1}{\eta} \|\w\|_2^2 + \frac{1}{\gamma}\Psi(\w)$ and $\Psi$ is a $\nu$-self concordant barrier for $\domainw$.

        \STATE Play $\action_t = \w_t + \big(\nabla^2 R(\w_t)\big)^{-1/2}\v_t~$, where $v_t$ is uniformly sampled from the unit sphere.

        \STATE Observe $\action_{\tau}^\top \bell_{\tau}$ for $\tau: \tau + d_\tau = t$.

        \STATE Compute loss estimators $\blhat_\tau = K \bell_\tau^\top \action_\tau \big(\nabla^2 R(\w_\tau)\big)^{1/2}\v_\tau$ for $\tau: \tau + d_\tau = t$ and update $\hat{L}_t$.

    \ENDFOR
\end{algorithmic}
\end{algorithm}

In this section, we show how to apply our analysis of FTRL to linear bandits with delayed feedback. %We start with an explanation of how this setting can be seen as a special case of our general setting. 
Linear bandits with delayed feedback is an instance of our general setting where $\bell_t \in \reals^K$ such that $\max_t \|\bell_t\|_2 \leq 1$, $\Aset = \domainw \subset \reals^K$, and the feedback function is $\observe(\bell, \action) = \bell^\top \action$. Additionally, we assume that $\domainw \subseteq \ball(\br)$, where $\ball(\br)$ is an $L_2$ ball with radius $\br$. 
%The delayed linear bandit setting proceeds in rounds $t = 1, \ldots, T$. In each round $t$, the learner predicts with $\tilde{\w}_t \in \domainw \subseteq \ball(\br)$, suffers loss $\tilde{\w}_t^\top \bell_t$, and observes $\tilde{\w}_\tau^\top \bell_\tau$ for $s: s + d_\tau = t$. 

Our algorithm for the linear bandit setting is inspired by \citet{abernethy2008competing}, who show a regularizer delivering nearly optimal bounds for the linear bandit setting with an efficient algorithm. 
For the delayed linear bandit setting we use a regularizer of the form 
$
    R(\w) = \frac{1}{\eta} \|\w\|_2^2 + \frac{1}{\gamma}\Psi(\w),
$
where $\Psi$ is a $\nu$-self-concordant barrier function for $\domainw$. Recall that a thrice-differentiable function $\Psi$ is called self-concordant if it is convex and satisfies
$
    |\nabla^3\Psi(\w)[\h,\h,\h]| \leq 2 \big(\nabla^2 \Psi(\w)[\h, \h]\big)^{3/2}~,
$
where 
$
    \nabla^3\Psi(\w)[\h_1,\h_2,\h_3] = \frac{\partial^3}{\partial t_1 \partial t_2 \partial t_3}|_{t_1 = t_2 = t_3 = 0} \Psi(\w + t_1 \h_1 + t_2 \h_2 + t_3 \h_3).
$
A self-concordant function $\Psi$ is a $\nu$-self-concordant barrier if 
$% \begin{align*}
    |\nabla \Psi(\w)[\h]| \leq \sqrt{\nu \nabla^2 \Psi(\w)[\h, \h]}~.
$ %\end{align*}
As a specific example, the log barrier, $-\log(x)$, is $1$-self-concordant for the non-negative reals.
For a thorough introduction to self-concordant barriers, we refer the reader to \citet{nesterov1994interior}. Here, we only recall the most important properties, which can be found in \citep[Section~2]{nemirovski2008interior}. 

The first property will allow us to satisfy the stability condition of the Hessian in Lemma~\ref{lem:main}:
For $\w, \w' \in \domainw$, we have that if $\|\w - \w'\|_{\Psi, \w}^\star < 1$ then
\begin{align}\label{eq:scbstability}
    \big(1 - \|\w - \w'\|_{\Psi, \w}^\star\big)^2 \nabla^2 \Psi(\w) \preceq \Psi(\w') \preceq \big(1 - \|\w - \w'\|_{\Psi, \w}^\star\big)^{-2}\nabla^2 \Psi(\w)~.
\end{align}
Next, given $\y \in \domainw$ denote by $\pi_\y(\x) = \inf\{z \geq 0: \y + z^{-1}(\x - \y) \in \domainw\}$ the Minkowsky function. We denote by $\domainw_\delta = \{\w: \pi_{\w_1}(\w) \leq (1 + \delta)^{-1}\}$, where $\delta > 0$. If $\Psi$ is a $\nu$-self-concordant barrier, then for any $\w \in \domainw_\delta$ 
\begin{align}\label{eq:scbbias}
    \Psi(\w) - \Psi(\w_1) \leq \nu \ln \big((1+\delta)\delta^{-1}\big)~.
\end{align}
This property will essentially allow us to show that for any benchmark point $\tilde \u\in\domainw$ there is a sufficiently close $u$ such that the penalty term in the regret (namely, $R(\u) - R(\w_1)$) is nicely bounded (see Eq.~\eqref{eq:thisbound} in the proof). Finally, we turn to the way we choose the played action $\action_t \in \actionset$ and the construction of the estimator. For that we use the fact that,
 $\dikin_\Psi(\w, 1) \subseteq \domainw = \actionset$ for any $\w \in \domainw$. Now, let $\v$ be in the $K$-dimensional unit sphere, denoted by $\sphere$. For any $\w \in \domainw$, we have that $\action = \w + \big(\nabla^2 \Psi(\w)\big)^{-1/2}\v \in \actionset$ because $\|\action - \w\|_{\Psi, \w} = 1$. 
For adversarial linear bandits with delayed feedback, we use $\action_t = \w_t + \big(\nabla^2 R(\w_t)\big)^{-1/2}\v_t$, 
where $\v_t$ is sampled i.i.d.\ from the uniform distribution over the unit sphere.
Note that $\action_t \in \domainw$ still holds.
As for the loss estimate, we use $\blhat_t = K \bell_t^\top \action_t \big(\nabla^2 R(\w_t)\big)^{1/2}\v_t$,
which can be seen to an unbiased estimator for $\bell_t$ after observing that $\E[\v_t\v_t^\top|\mathcal{F}_t] = \frac{1}{K}\I$. 

% \begin{restatable}{retheorem}{linbanth}
%     Suppose that $\|\bell_t\|_2 \leq 1$ for all $t\in [T]$.
%     % and that $\sumT |\miss_t|\leq D$.
%     Let
%     \[
%     \eta = \min \left\{\sqrt{\frac{B^2}{16 D}}, \frac{1}{(16\dmax )^2}\right\},\;\gamma = \min\left\{\sqrt{\frac{\nu \ln\big(1 + \sqrt{T}\big)}{16 (BK)^2 T}}, \frac{1}{(64BK(1+\dmax))^2}\right\}~.
%     \]
%     With predictions \eqref{eq:linbandpred} and loss estimator \eqref{eq:linbandest} we have that, for any $\u \in \domainw$
%     %
%     \begin{align*}
%         \E\left[\sumT(\tilde{\w}_t - \u)^\top\bell_t \right] &\leq 8 B\sqrt{D} + 8 BK\sqrt{\nu T\ln\big(1 + \sqrt{T}\big)} + 2 B\sqrt{T} \\
%         &\quad + 256 \dmax^2 + \big(64BK(1+\dmax)\big)^2 \nu \ln\big(1 + \sqrt{T}\big)~.
%     \end{align*}
%     %
% \end{restatable}
\begin{restatable}{retheorem}{linbanth}\label{th:linban}
    Let $\u \in \domainw$.
    Running Delayed FTRL for linear bandits (Algorithm~\ref{alg:linban}) with $\gamma = \min\left\{\frac{1}{(64BK(1+\dmax))^2},\sqrt{\frac{\nu \ln\big(1 + \sqrt{T}\big)}{16 (BK)^2 T}}\right\}$ and $\eta = \min \left\{\frac{1}{(16\dmax )^2},\sqrt{\frac{B^2}{16 D}}\right\}$ guarantees,
    $$
        \E\left[\sumT(\action_t - \u)^\top\bell_t \right] \leq 14 BK\sqrt{\nu T\ln (T)} +8 B\sqrt{D} + 2^{14} \nu B^2 K^2 (1+\dmax)^2 \ln (T).
    $$
\end{restatable}

The proof is similar to the proof of Theorem~\ref{th: CMAB regret} and follows from an application of Corollary~\ref{cor:main}, which we can apply after verifying its assumptions (see the full proof in Appendix~\ref{sec:appendix-linear-bandits}).

Let us interpret the result of Theorem~\ref{th:linban}. It is known that a $K$-self-concordant barrier exists for all convex and closed $\domainw$ \citep{bubeck2015entropic, chewi2021entropic}. Therefore, for $d_t = d$, our algorithm always guarantees a $O(K^{3/2}B\sqrt{T\ln(T)}) + B\sqrt{dT})$ bound for delayed bandit linear optimization. This makes our bound slightly suboptimal, as \citet{ito2020delay} show that the minimax regret is $\Theta(KB\sqrt{T\ln(T)}) + B\sqrt{dT})$. This trade-off between running time and regret bounds also exists in non-delayed linear bandits: algorithms that obtain the optimal regret bound \citep{bubeck2015entropic, hazan2016volumetric, vanderhoeven2018many, ito2020tight, zimmert2022return} have running time polynomial in both $T$ and $K$, whereas slightly suboptimal algorithms in terms of regret, such as Scrible \citep{abernethy2008competing}, have $O(K^3)$ runtime, given that a self-concordant barrier for the set of interest can be efficiently computed. For particular action sets, there exist algorithms that avoid this trade-off, see \citep{bubeck2012towards}. 
There also exist domains for which $\nu$ can be considerably smaller. For example, $\Psi(\w) = -\log(1 - \|\w\|_2^2)$ is a 1-self-concordant barrier for the $L_2$ unit ball, in which case our bound is $O(KB\sqrt{T\ln(T)}) + B\sqrt{dT})$. This matches the upper bound of \citet{ito2020delay} with less stringent assumptions, since \citet{ito2020delay} assume constant delays whereas Algorithm~\ref{alg:linban} can handle variable delays. 

\paragraph{Efficient implementation.} \citet[section 9]{abernethy2008competing} provide an approximation of FTRL with self-concordant barrier regularizers with a $O(K^3)$ per-round running time. Seemingly this implies that we can not use the results of \citet{abernethy2008competing} to approximate our algorithm since our regularizer $R$ is not a barrier, but only a self-concordant function on $\domainw$---see also 
\citep[Proposition~2.1.1]{nemirovski2004interior}. However, even though \citet{abernethy2008competing} assume that the regularizer is a self-concordant barrier, the properties used to prove that the approximation has a suitable regret bound rely on properties of self-concordant functions---see \citep[Chapter~2, Statement~IX]{nemirovski2004interior}. Thus, we can use the approximation of \citet{abernethy2008competing} to obtain a $O(K^3)$ per round running time approximation of our algorithm.%Similarly, the algorithm of \citet{huang2023banker} can also be implemented in $O(K^3)$ per-round runtime. However, if run with an $O(K)$-self concordant barrier their algorithm guarantees a $O(K^{3/2}B\sqrt{T\ln(T)}) + K^2 B\sqrt{D})$ regret bound, which has a significantly worse dependency on $K$ compared to our result. 

\section{Future Work and Discussion} 

We provided a new analysis of FTRL with delayed bandit feedback which leads to several new results for combinatorial semi-bandits, adversarial Markov decision processes, and linear bandits. The main downside of our approach is the additive $\dmax^2 \log(T)$ term in our bounds. Even though it is a lower order term, it prevents us from employing the skipping technique of \citet{thune2019nonstochastic}. Therefore, an important open problem is removing the $\dmax^2 \log(T)$ term and replacing it with $\dmax \log(T)$ or removing it altogether. 
Another direction for future work is MDPs
%, where we assumed that the transition function was known. Extending our ideas to the MDP setting
with unknown transitions.
Extending our ideas to that setting is not straightforward due the fact that standard analyses for this prwoblem use a changing domain $\domainw$.
% employed in the standard approach for this setting.
Finally, in the linear bandits setting it would be interesting to see whether the results for non-delayed feedback and the specific choice of $\domainw$ used by \citet{bubeck2012towards} can be transferred to the delayed feedback setting.

% Acknowledgments---Will not appear in anonymized version
\acks{This work was mostly done while DvdH was at the University of Milan partially supported by the MIUR PRIN grant Algorithms, Games, and Digital Markets (ALGADIMAR) and partially supported by Netherlands Organization for Scientific Research (NWO), grant number VI.Vidi.192.095. LZ and NCB are partially supported by the EU Horizon 2020 ICT-48 research and innovation action under grant agreement 951847, project ELISE (European Learning and Intelligent Systems Excellence) and by the FAIR (Future Artificial Intelligence Research) project, funded by the NextGenerationEU program within the PNRR-PE-AI scheme (M4C2, investment 1.3, line on Artificial Intelligence).}

\DeclareRobustCommand{\VAN}[3]{#3} 
\bibliography{sample}
\DeclareRobustCommand{\VAN}[3]{#2}

\newpage

\appendix

\section{Deferred Proofs of Analysis (Section~\ref{sec:analysis})}
\label{appendix:proofs-for-general-analysis}

\lemmamain*
\begin{proof}
%By Lemma~\ref{lem:windikin}, as 
%
%\begin{align*}
%    \sum_{\tau \in m_t} \|\bell_\tau\|_{R, \w_t} \leq \frac{1}{8}, && \sum_{\tau \in m_t} \|\blhat_\tau\|_{R, \w_t} \leq \frac{1}{8},
%\end{align*}
%
%holds by assumption \tal{By assumption $\|\hat \ell_t\|_{\w_t} \lesssim 1/{d_{max}}$ not $\|\hat \ell_\tau \|_{\w_t}$ shouldn't we apply Lemma 5 first?} we know that $\w(\bar{\L}_t^m), \w(\Lhat_t^m) \in \dikin_R(\w(\Lhat_t, \half))$.
% We start by showing that $\sum_{\tau \in [t] \setminus \obs_t} \|\blhat_\tau\|_{R, \w_t} \leq \frac{1}{16}$ for all $t$.
 By Lemma~\ref{lem:induction}, $\w_t \in \dikin_R(\w_\tau, \half)$ for all $\tau \in [t] \setminus \obs_t = \miss_t \cup \{t\} \subseteq  \{t - \dmax, \dots, t\}$ and all $t$. We can use this to show
% \begin{align*}
%     \sum_{\tau \in [t] \setminus \obs_t} \|\blhat_\tau\|_{R, \w_t} \leq 2 \sum_{\tau \in [t] \setminus \obs_t} \|\blhat_\tau\|_{R, \w_\tau} \leq \frac{1}{16},
% \end{align*}
% by employing $4 \nabla^2 R( \w_\tau) \succeq \nabla^2 R( \w_t)$ as well as the assumption on $ \|\blhat_t\|_{R, \w_t}$. It is clear that
% \begin{align*}
%     \sum_{\tau \in \miss_t} \|\blhat_\tau\|_{R, \w_t} \leq \sum_{\tau \in [t] \setminus \obs_t} \|\blhat_\tau\|_{R, \w_t} \leq \frac{1}{16},
% \end{align*}
% and through the assumption on $\|\bell_t\|_{R, \w}$ we can conclude that $\sum_{\tau \in m_t} \|\bell_\tau\|_{R, \w_t} \leq \frac{1}{16}$.
\begin{align*}
    \sum_{\tau \in \miss_t} \|\blhat_\tau\|_{R, \w_t} \le \sum_{\tau \in [t] \setminus \obs_t} \|\blhat_\tau\|_{R, \w_t} \leq 2 \sum_{\tau \in [t] \setminus \obs_t} \|\blhat_\tau\|_{R, \w_\tau} \leq \frac{1}{16},
\end{align*}
by employing $4 \nabla^2 R( \w_\tau) \succeq \nabla^2 R( \w_t)$ and the assumption on $\|\blhat_t\|_{R,\w_t}$.
By Lemma~\ref{lem:windikin} and our assumptions on $\|\blhat_t\|_{R,\w_t}$ and $\|\bell_t\|_{R,\w_t}$, we can thus conclude that $\w(\bar{\L}_t^m), \w(\Lhat_t^m), \w(\Lstar_t) \in \dikin_R(\w_t, \half)$.
By H\"older's inequality and Lemma~\ref{lem:FTRLnorm} 
\begin{align*} 
    \E\Big[(\w_t - \w(\bar{\L}_t^m))^\top \bell_t\Big]
&\leq
    \E\Big[\|\w_t - \w(\bar{\L}_t^m)\|_{R, \w_t}^\star\|\bell_t\|_{R, \w_t}\Big]
\\&\leq
    \E\Big[8\|\Lhat_t - \bar{\L}_t^m\|_{R, \w_t}\|\bell_t\|_{R, \w_t}\Big]
\\&=
    \E\left[8\bigg\|\sum_{\tau \in m_t} \bell_\tau\bigg\|_{R, \w_t}\|\bell_t\|_{R, \w_t}\right]
\leq
    8\fbound^2|m_t|, \numberthis \label{eq:intermediatemainbound1}
\end{align*}
where the last inequality is due to the triangle inequality and the assumptions on $\|\bell_\tau\|_{R,\w}$. Similarly we bound
\begin{align*}
    \E\Big[(\w(\bar{\L}_t^m) - \w(\Lhat_t^m))^\top \bell_t\Big] &\leq 8\fbound\E\left[\bigg\|\sum_{\tau \in m_t}(\bell_\tau - \blhat_\tau)\bigg\|_{R, \w_t} \right] \\
    \E\Big[(\w(\Lhat_t^m) - \w(\Lstar_t))^\top \blhat_t\Big] &\leq 8\bbound^2 \numberthis \label{eq:intermediatemainbound2}.
\end{align*}
%
% \begin{align} \label{eq:intermediatemainbound}
%     & \E\big[(\w_t - \w(\bar{\L}_t^m))^\top \bell_t\big] + \E\big[(\w(\bar{\L}_t^m) - \w(\Lhat_t^m))^\top \bell_t\big] + \E\big[(\w(\Lhat_t^m) - \w(\Lstar_t))^\top \blhat_t\big] \nonumber\\
%     & \leq \E\big[\|\w_t - \w(\bar{\L}_t^m)\|_{R, \w_t}^\star\|\bell_t\|_{R, \w_t}\big] + \E\big[\|\w(\bar{\L}_t^m) - \w(\Lhat_t^m)\|_{R, \w_t}^\star \|\bell_t\|_{R, \w_t}\big]\nonumber \\
%     & \quad + \E\big[\|\w(\Lhat_t^m) - \w(\Lstar_t)\|_{R, \w_t}^\star \|\blhat_t\|_{R, \w_t}\big] \nonumber\\
%     & \leq \E\big[8\|\Lhat_t - \bar{\L}_t^m\|_{R, \w_t}\|\bell_t\|_{R, \w_t}\big] + \E\big[8\|\bar{\L}_t^m - \Lhat_t^m\|_{R, \w_t} \|\bell_t\|_{R, \w_t}\big] \nonumber\\
%     & \quad + \E\big[8\|\blhat_t\|_{R, \w_t}^2\big] \nonumber\\
%     & = \E\big[8\|\sum_{\tau \in m_t} \bell_\tau\|_{R, \w_t}\|\bell_t\|_{R, \w_t}\big] + \E\big[8\|\sum_{\tau \in m_t}(\bell_\tau - \blhat_\tau)\|_{R, \w_t} \|\bell_t\|_{R, \w_t}\big] \nonumber\\
%     & \quad + \E\big[8\|\blhat_t\|_{R, \w_t}^2\big] \nonumber \\
%     & \leq 8\fbound^2|m_t| + 8\bbound^2 + 8\fbound\E\big[\|\sum_{\tau \in m_t}(\bell_\tau - \blhat_\tau)\|_{R, \w_t} \big] 
% \end{align}
%
By Lemma~\ref{lem:BTL} we have that 
\begin{align}\label{eq:mainintermediate3}
    \sumT (\w(\Lstar_t) - \u)^\top \blhat_t \leq R(\u) - R(\w_1).
\end{align}
Thus, by combining equations \eqref{eq:decomp}, \eqref{eq:intermediatemainbound1}, \eqref{eq:intermediatemainbound2}, and \eqref{eq:mainintermediate3},
\begin{align*}
    & \sumT \E\big[(\w_t - \u)^\top \bell_t \big] = \sumT \E\big[(\w(\Lstar_t) - \u)^\top \blhat_t\big] - 
    \sumT \E\big[(\w(\Lhat_t^m) - \u)^\top \bias_t\big]\\
    & + \sumT\Big( \E\big[(\w_t - \w(\bar{\L}_t^m))^\top \bell_t\big] + \E\big[(\w(\bar{\L}_t^m) - \w(\Lhat_t^m))^\top \bell_t\big] + \E\big[(\w(\Lhat_t^m) - \w(\Lstar_t))^\top \blhat_t\big] \Big) \\
    & \leq R(\u) - R(\w_1) + 8\bbound^2T + \sumT \left(8\fbound^2|m_t| + 8\fbound\E\left[\bigg\|\sum_{\tau \in m_t}(\bell_\tau - \blhat_\tau)\bigg\|_{R, \w_t}\right]\right)\\
    & \quad -\sumT \E\big[(\w(\Lhat_t^m) - \u)^\top \bias_t\big]~,
\end{align*}
which concludes the proof. 
\end{proof}

\corollarymain*
\begin{proof}
We are looking to control $\E\Big[\big\|\sum_{\tau \in \miss_t}(\bell_\tau - \blhat_\tau)\big\|_{R,\w_t}\Big]$ for a given $t \in [T]$. We start by considering
\begin{align*}
    % & \E\Big[\Big(\sum_{\tau \in m_t}(\bell_\tau - \blhat_\tau)\Big)^\top(\nabla^2 R(\w(\Lhat_t)))^{-1}\Big(\sum_{\tau \in m_t}(\bell_\tau - \blhat_\tau)\Big) \Big]\\
    % & = \sum_{\tau \in m_t}\E[(\bell_\tau - \blhat_\tau)^\top(\nabla^2 R(\w(\Lhat_t)))^{-1}(\bell_\tau - \blhat_\tau)] \\
    % & = \sum_{\tau \in m_t}\E[\blhat_\tau^\top(\nabla^2 R(\w(\Lhat_t)))^{-1}\blhat_\tau] - \E[\bell_\tau^\top(\nabla^2 R(\w(\Lhat_t)))^{-1}\bell_\tau] \\
    % & \leq \sum_{\tau \in m_t}\big(\E[\blhat_\tau^\top(\nabla^2 R(\w(\Lhat_t)))^{-1}\blhat_\tau],
    \E\left[ \bigg\|\sum_{\tau \in \miss_t}(\bell_\tau - \blhat_\tau)\bigg\|_{R,\w_t}^2 \right]
& =
    \sum_{\tau \in m_t}\E\Big[\big\|\bell_\tau - \blhat_\tau\big\|_{R,\w_t}^2\Big]
\\& =
    \sum_{\tau \in m_t} \left( \E\Big[\big\|\blhat_\tau\big\|_{R,\w_t}^2\Big] - \E\Big[\big\|\bell_\tau\big\|_{R,\w_t}^2\Big] \right)
\\& \leq
    \sum_{\tau \in m_t} \E\Big[\big\|\blhat_\tau\big\|_{R,\w_t}^2\Big]~,
\end{align*}
where we used that $\E\Big[(\blhat_\tau - \bell_\tau)^\top\big(\nabla^2 R(\w(\Lhat_t))\big)^{-1} (\blhat_{\tau'} - \bell_{\tau'})\,|\,\mathcal{F}_t\Big] = 0$ for $\tau \neq \tau'$ in the first equality, and that $\E[\blhat_\tau\,|\,\mathcal{F}_t] = \bell_\tau$ in the second equality. 
In turn, the above together with Jensen's inequality implies that 
\begin{align}\label{eq:variancebound}
     \E\left[\bigg\|\sum_{\tau \in m_t}(\bell_\tau - \blhat_\tau)\bigg\|_{R, \w_t}\right] 
& \leq
    \sqrt{\sum_{\tau \in m_t}\E\Big[\big\|\blhat_\tau\big\|_{R, \w_t}^2\Big]} \nonumber \\
    %& \leq  \sqrt{2\sum_{\tau \in m_t}\E[\|\bell_\tau - \blhat_\tau\|_{R, \w_t}^2]} \\
& \leq
    \sqrt{4\sum_{\tau \in m_t}\E\Big[\big\|\blhat_\tau\big\|_{R, \w_{\tau}}^2\Big]} 
\leq
    \sqrt{4 |m_t|\bbound^2}~, 
\end{align}
where in the second inequality we used Lemma~\ref{lem:induction} together with $4 \nabla^2 R( \w) \succeq \nabla^2 R( \w') \succeq \frac{1}{4} \nabla^2 R( \w)$ for all $\w \in \domainw$ and $\w' \in \dikin_R(\w, \half)$. Finally, the third inequality of~\ref{eq:variancebound} is due to the assumptions of Lemma~\ref{lem:main}. We conclude by substituting this bound into the results of Lemma~\ref{lem:main},
\begin{align*}
    \sumT \E\big[(\w_t - \u)^\top \bell_t \big] & \leq R(\u) - R(\w_1) + 8\bbound^2T  + \sumT \Big(8\fbound^2|m_t| + 16\sqrt{ |m_t|\fbound^2 \bbound^2}\Big) \\
    & \leq  R(\u) - R(\w_1) + 16\bbound^2 T + 16\fbound^2 \sumT |m_t|,
\end{align*}
where in the last inequality we used that $\sqrt{ab} \leq \half (a + b)$ for $a, b > 0$. 
\end{proof}

\lemmaFTRLnorm*
\begin{proof}
    By Taylor's theorem and the optimality of $\w(\L')$ we have that for some $\zeta$ on the line segment between $\w(\L')$ and $\w(\L)$
    \begin{align*}
        & {\L'}^\top \w(\L) + R(\w(\L)) - {\L'}^\top \w(\L') - R(\w(\L')) \\
        & \geq \frac{1}{2} (\w(\L') - \w(\L))^\top \nabla^2 R(\zeta) (\w(\L') - \w(\L)) \\
        & \geq \frac{1}{8}(\w(\L') - \w(\L))^\top \nabla^2 R(\x) (\w(\L') - \w(\L)),
    \end{align*}
    where the last inequality is due the assumption on $\nabla^2 R(\w)$, which is applicable because if $\w(\L'), \w(\L) \in \dikin_R(\x, \half)$ it implies that the line segment between $\w(\L')$ and $\w(\L)$ is also in $\dikin_R(\x, \half)$. Thus $\zeta \in \dikin_R(\x, \half)$.
    
    By Taylor's theorem we have that 
    \begin{align*}
        & {\L'}^\top \w(\L) + R(\w(\L)) - {\L'}^\top \w(\L') - R(\w(\L')) \\
        & = (\L' - \L)^\top(\w(\L) - \w(\L')) + {\L}^\top \w(\L) + R(\w(\L)) - {\L}^\top \w(\L') - R(\w(\L'))\\
        &  \leq (\L' - \L)^\top(\w(\L) - \w(\L')) \\
        & \leq \|\L' - \L\|_{R, \x} \|\w(\L) - \w(\L')\|_{R, \x}^\star~,
    \end{align*}
    where the first inequality is due to the optimality of $\w(\L)$ and the second inequality is H\"older's inequality. Thus, we may conclude that 
    \begin{align*}
         \|\L' - \L\|_{R, \x} \|\w(\L) - \w(\L')\|_{R, \x}^\star \geq \frac{1}{8} \Big(\|\w(\L) - \w(\L')\|_{R, \x}^\star\Big)^2~,
    \end{align*}
    which concludes the proof after multiplying both sides of the above inequality by $\frac{8}{\|\w(\L) - \w(\L')\|_{R, \x}^\star}$.
\end{proof}

\begin{restatable}{relemma}{lemmawindikin}
\label{lem:windikin}
    Suppose that $\nabla^2 R( \w') \succeq \frac{1}{4} \nabla^2 R( \w) $ for all $\w \in \domainw$ and $\w' \in \dikin_R(\w, \half)$ and for $R$ strictly convex and twice differentiable. Let $z \subset \nats$ be a finite set, and define $\L' = \L + \sum_{\tau \in z} \y_\tau$, where $\y_\tau \in \reals^K$. If $\sum_{\tau \in z}\|\y_\tau\|_{R, \w(\L)} \leq \tfrac{1}{16}$, then $\w(\L') \in \dikin_R(\w(\L), \half)$.
\end{restatable}
\begin{proof}
Because of the strict convexity of $R$, to show that $\w(\L') \in \dikin_R(\w(\L), \half)$ it suffices to show that for all $\x$ on the boundary of $\dikin_R(\w(\L), \half)$
\begin{align}\label{eq:xbigger}
    {\L'}^\top \x + R(\x) \geq {\L'}^\top \w(\L) + R(\w(\L)).
\end{align}
To see why the strict convexity of $R$ is sufficient, suppose that all $\x$ on the boundary of $\dikin_R(\w(\L), \half)$ indeed satisfy equation~\eqref{eq:xbigger}. For the sake of contradiction suppose that $\w(\L')$ is not in $\dikin_R(\w(\L), \half)$. Let $\z = (1-a)\w(\L) + a \w(\L')$ be the point on the boundary of $\dikin_R(\w(\L), \half)$ on the segment between $\w(\L)$ and $\w(\L')$. Then
\begin{align*}
    {\L'}^\top \w(\L) + R(\w(\L)) & \leq {\L'}^\top \z + R(\z) \\
    & < (1 - a)({\L'}^\top \w(\L) + R(\w(\L))) + a({\L'}^\top \w(\L') + R(\w(\L'))) \\
    & \leq {\L'}^\top \w(\L) + R(\w(\L))~ ,
\end{align*}
where the last inequality is by definition of $\w(\L')$. Thus, we have a contradiction, which implies that if all $\x$ on the boundary of $\dikin_R(\w(\L), \half)$ satisfy equation~\eqref{eq:xbigger} then $\w(\L') \in \dikin_R(\w(\L), \half)$.

We proceed by showing that all $\x$ on the boundary of $\dikin_R(\w(\L), \half)$ satisfy equation~\eqref{eq:xbigger}. Denote by $\h = \x - \w(\L)$. Using Taylor's theorem, there exists $\zeta$ on the segment between $\x$ and $\w(\L)$ such that 
\begin{equation}\label{eq:taylor}
\begin{split}
    & {\L'}^\top \x + R(\x) - {\L'}^\top \w(\L) - R(\w(\L)) \\
    & = (\L' - \L)^\top\h + (\L + \nabla R(\w(\L)))^\top \h + \half \h^\top \nabla^2 R(\zeta) \h \\
    & \geq  (\L' - \L)^\top\h + \half \h^\top \nabla^2 R(\zeta) \h \\
    & \geq  (\L' - \L)^\top\h + \tfrac{1}{8} \h^\top \nabla^2 R(\w(\L)) \h
\end{split}
\end{equation}
where the first inequality is due to the optimality of $\w(\L)$ and the second inequality is because $\zeta,\w(\L)\in \dikin_R(\w, \half)$. 
Thus, by applying H\"older's inequality we can see that
\begin{align*}
    & {\L'}^\top \x + R(\x) - {\L'}^\top \w(\L) - R(\w(\L)) \\
    & \geq  (\L' - \L)^\top\h + \tfrac{1}{8} \h^\top \nabla^2 R(\w(\L)) \h \\
    & \geq -\sum_{\tau \in z}\|\y_\tau\|_{R, \w(\L)}\|\h\|_{R, \w(\L)}^\star + \tfrac{1}{8} \h^\top \nabla^2 R(\w(\L)) \h \\ 
    & = -\frac{1}{2} \sum_{\tau \in z}\|\y_\tau\|_{R, \w(\L)} + \tfrac{1}{32} \\
    & \geq 0
\end{align*}
where the equality is due to the fact that $\|\h\|_{R, \w(\L)}^\star = \half$, as $\x$ is on the boundary of $\dikin(\w(\L), \half)$ and the final inequality is due to the assumption that $\sum_\tau\|\y_\tau\|_{R, \w(\L)} \leq \tfrac{1}{16}$.
\end{proof}

\begin{restatable}{relemma}{lemmainduction}\label{lem:induction}
    Suppose that $4 \nabla^2 R( \w) \succeq \nabla^2 R( \w') \succeq \frac{1}{4} \nabla^2 R( \w)$ for all $\w \in \domainw$ and $\w' \in \dikin_R(\w, \half)$. Also suppose that $\|\blhat_t\|_{R, \w_t} \leq \frac{1}{64(1 + \dmax)}$  for all $t$ with probability 1.
    % and let $R$ be convex. %Also suppose that, for all $t$, $\|\blhat_\tau\|_{R, \w(L_t)} \leq \frac{1}{8(1 + \dmax)}$ for all $\w(\Lhat_\tau) \in \dikin_R(\w(\Lhat_t), \half)$.
    Then, for all $t$ and all $\delta \in [\dmax + 1]$, we have that $\w_{t+\delta} \in \dikin_R(\w_t, \half)$.
\end{restatable}
\begin{proof}
We prove the statement by induction on $t$.
For the base case, we need to show that for all $\delta \in [1 + \dmax]$, $\w_{1+\delta} \in \dikin_R(\w_1, \half)$. We start by showing that $\w_2 \in \dikin_R(\w_1, \half)$ where $\Lhat_2 = \Lhat_1 + \blhat_1\indicator\{1 \in o_2\}$. Since $\|\blhat_1\|_{R, \w_1} \leq \frac{1}{16}$ by assumption, this follows from Lemma~\ref{lem:windikin}. Now, $\Lhat_3 = \Lhat_1 + \blhat_1\indicator\{1 \in o_3\setminus o_2\} + \blhat_2\indicator\{2\in o_3\}$. Since $4 \nabla^2 R( \w) \succeq \nabla^2 R( \w') \succeq \frac{1}{4} \nabla^2 R( \w) $ for all $\w \in \domainw$ and $\w' \in \dikin_R(\w, \half)$ and $\|\blhat_t\|_{R, \w_t} \leq \frac{1}{64(1 + \dmax)}$ for all $t$, we have that 
\begin{align*}
    \|\blhat_1\|_{R, \w_1} + \|\blhat_2\|_{R, \w_1} \leq \|\blhat_1\|_{R, \w_1} + 2 \|\blhat_2\|_{R, \w_2} \leq \frac{1}{16}
\end{align*}
and Lemma~\ref{lem:windikin} implies that $\w_3 \in \dikin_R(\w_1, \half)$. We can now repeat this argument to show that $\w_{1 + \delta} \in \dikin_R(\w_1, \half)$ for $\delta \in [\dmax + 1]$, which completes the proof for the base case.

For the induction step, assume that $\w_t \in \dikin_R(\w_\tau,\half)$ for $\tau\in\{t-\dmax-1, \ldots, t-1\}$.
Recall that $\Lhat_{t+1} = \Lhat_t + \sum_{\tau \in \obs_{t+1} \setminus \obs_t} \blhat_{\tau}$.
Since $\obs_{t+1} \setminus \obs_t \subseteq \{t-\dmax, \ldots, t-1, t\}$, we have that
\begin{align*}
     \sum_{\tau \in \obs_{t+1} \setminus \obs_t} \|\blhat_{\tau}\|_{R, \w_t} \leq  2\sum_{\tau \in \obs_{t+1} \setminus \obs_t} \|\blhat_{\tau}\|_{R, \w_\tau} \leq \frac{|\obs_{t+1} \setminus \obs_t|}{32(\dmax + 1)} \leq \frac{1}{32}~,
\end{align*}
where in the first inequality we used that $\nabla^2 R(\w_t) \succeq \frac{1}{4} \nabla^2 R(\w_\tau)$ if $\w_t \in \dikin_R(\w_\tau, \half)$ which for $\tau \in \{t-\dmax, \ldots, t-1\}$ is true by the inductive assumption, and in the second inequality we used the assumption $\|\blhat_\tau\|_{R, \w_\tau} \leq \frac{1}{64(1 + \dmax)}$.
Then Lemma~\ref{lem:windikin} implies $\w_{t+1} \in \dikin_{R}(\w_t, \half)$ which in turn implies
$\nabla^2 R(\w_t) \succeq \frac{1}{4} \nabla^2 R(\w_{t+1})$. Using this last inequality, we can see that
\begin{align*}
    \sum_{\tau \in \obs_{t+2} \setminus \obs_t} \|\blhat_{\tau}\|_{R, \w_t} & \leq \sum_{\tau \in  \obs_{t+2} \setminus (\obs_t \cup \{t+1\})} \|\blhat_{\tau}\|_{R, \w_t} + \|\blhat_{t + 1}\|_{R, \w_t} \\
    & \leq  \frac{1}{32} + 2 \|\blhat_{t + 1}\|_{R, \w_{t+1}} \leq \frac{1}{16}~.
\end{align*}
Thus, by Lemma~\ref{lem:windikin}, $\w_{t+2} \in \dikin_R(\w_t, \half)$. We can repeat this argument to show that for all $\delta \in [1 + \dmax]$, $\w_{t+\delta} \in \dikin_R(\w_t, \half)$.

% By assumption $\|\blhat_\tau\|_{R, \w(L_t)} \leq \frac{1}{8(1 + \dmax)}$ if $\w(\Lhat_\tau) \in \dikin_R(\w(\Lhat_t), \half)$. Using this assumption we will show that $\w(\Lhat_\tau) \in \dikin_R(\w(\Lhat_t), \half)$ for $\tau \in \{\tau:|t-\tau| < \dmax + 1\}$ by induction. 

% Assume that for all $t' < t$, for $\tau \in \{\tau:|t'-\tau| < \dmax + 1\}$ we have that $\w(\Lhat_{\tau}) \in \dikin_R(\w(\Lhat_{t'}), \half)$ and that  $\w(\Lhat_\tau) \in \dikin_R(\w(\Lhat_t), \half)$ for all $\tau$ such that $t - \tau \in [0, 1+\dmax]$. Denote by $\arrive_t = \{\tau:\tau + d_\tau = t\}$ the indices of the losses that are observed by the learner in round $t$. We have that 
% %
% \begin{align*}
%     \sum_{\tau \in \arrive_t} \|\blhat_\tau\|_{R, \w_t} \leq \frac{|\arrive_t|}{8(1 + \dmax)} \leq \frac{1}{8}
% \end{align*}
% %
% which means that by Lemma~\ref{lem:windikin} $w(\Lhat_{t+1}) \in \dikin_R(\w(\Lhat_t), \half)$ since $\Lhat_{t+1} = \Lhat_t + \sum_{\tau \in \arrive_t} \blhat_\tau$. Thus, we can now conclude that $\w(\Lhat_\tau) \in \dikin_R(\w(\Lhat_t), \half)$ for all $s$ such that $t - \tau \in [-1, 1+\dmax]$. We can repeat this argument to conclude that $\w(\Lhat_\tau) \in \dikin_R(\w(\Lhat_t), \half)$ for $\tau \in \{s:|t-s| < \dmax + 1\}$. 
% As for the base case, we only need to show that $\w(\L_2) \in \dikin_R(\w(\L_1), \half)$, which by Lemma~\ref{lem:windikin} is the case when $\|\blhat_1\|_{R,\w(\L_1)} \leq \frac{1}{8}$, which holds by assumption since $\w(\L_1) \in \dikin_R(\w(\L_1), \half)$.
\end{proof}

\begin{lemma}[Be-The-Leader Lemma]\label{lem:BTL}
For any fixed $\u \in \domainw$ we have that
\begin{align*}
    \sumT \blhat_t^\top (\w(\Lstar_t) - \u) &\leq R(\u) - R(\w_1) 
\end{align*}
\end{lemma}

\begin{proof}
We will prove the statement by induction on $T$. The base case holds by definition of $\w_1$. For the induction step, assume that
\begin{align*}
    \sum_{t=1}^{T-1} \blhat_t^\top \w(\Lstar_t) + R(\w_1) \leq \sum_{t=1}^{T-1} \blhat_t^\top \w + R(\w)
\end{align*}
for any $\w \in \domainw$. Adding $\blhat_T^\top\w(\Lstar_t)$ to both sides of the above inequality and setting $\w = \w(\Lstar_t)$ on the right-hand side of the above inequality we find
\begin{align*}
    \sum_{t=1}^{T} \blhat_t^\top \w(\Lstar_t) + R(\w_1) & \leq \argmin_{\w \in \domainw } \sum_{t=1}^{T} \blhat_t^\top \w + R(\w) \\
    & \leq  \sum_{t=1}^{T} \blhat_t^\top \u + R(\u),
\end{align*}
which proves the statement after reordering the above inequality.
\end{proof}

\newpage
\section{Deferred Proof for Combinatorial Semi-Bandits (Section~\ref{sec: CMAB})}\label{app: CMAB}
\thCMAB*
\begin{proof}
    We are looking to apply Corollary~\ref{cor:main}. We start by showing that indeed $4 \nabla^2 R( \w) \succeq \nabla^2 R( \w') \succeq \frac{1}{4} \nabla^2 R( \w)$ for all $\w \in \domainw$ and $\w' \in \dikin_R(\w, \half)$. With $\phi$, as defined in Lemma~\ref{lem:dikin implies factor 2},
    \begin{align*}
        \left(\nabla^2 R(\w)\right)(i, i) = \frac{1}{\gamma \w^2(i)} + \frac{1}{\eta \w(i)} \geq \frac{1}{\gamma \w^2(i)} = \left(\nabla^2 \phi(\w)\right)(i, i)~,
    \end{align*} 
    and we conclude that $\nabla^2 R(\w) \succeq \nabla^2 \phi(\w)$ as both matrices are diagonal. Together with $\domainw \subseteq \reals_+^\actiondim$, $R$ being strictly convex, and $\gamma \in (0, 1)$, we are in a position to apply Lemma~\ref{lem:dikin implies factor 2} to conclude that $\frac{1}{2} \w(i)\leq \w'(i) \leq  2 \w(i)$ for all $\w' \in \dikin_R(\w,\frac{1}{2})$ and all $i \in [\actiondim]$. We can use this fact to show that
    \begin{align*}
        \left(\nabla^2 R(\w')\right)(i, i) = \frac{1}{\gamma \w'^2(i)} + \frac{1}{\eta \w'(i)}
        \geq \frac{1}{\gamma 4 \w^2(i)} + \frac{1}{\eta 2 \w(i) }
        \geq \frac{1}{4} \left(\nabla^2 R(\w)\right)(i, i)
    \end{align*}
    and 
    \begin{align*}
        \left(\nabla^2 R(\w')\right)(i, i) = \frac{1}{\gamma \w'^2(i)} + \frac{1}{\eta \w'(i)}
        \leq \frac{1}{\gamma \frac{1}{4} \w^2(i)} + \frac{1}{\eta \frac{1}{2} \w(i) }
        \leq 4 \left(\nabla^2 R(\w)\right)(i, i)~.
    \end{align*} 
    With $\nabla^2 R(\w)$ being a diagonal matrix, we can conclude that $4 \nabla^2 R( \w) \succeq \nabla^2 R( \w') \succeq \frac{1}{4} \nabla^2 R( \w)$ for all $\w \in \domainw$ and $\w' \in \dikin_R(\w, \half)$.
    
    As the next step, we find an appropriate $\fbound$ such that $\|\bell_t\|_{R, \w} \leq \fbound \leq \frac{1}{16\dmax}$. For that we use $\sumdim \w(i) \leq \combm$, $\|\bell\|_\infty \leq 1$, and $\eta + \gamma \w(i) \geq \gamma \w(i)$ to bound 
    \begin{align*}
        \|\bell_t\|_{R, \w} = \sqrt{ \sumdim \bell_t(i)^2 \frac{\eta \gamma \w^2(i)}{\eta + \gamma \w(i)}} \leq \underbrace{\sqrt{\eta \combm}}_{\fbound}~. \numberthis \label{eqn:CMABfbound}
    \end{align*}
    and it is clear that $\fbound \leq \frac{1}{16\dmax}$, by $\eta \leq \frac{1}{16^2 \combm \dmax^2}$. Next is $\E \left[ \|\blhat_t\|_{R, \w_t}^2 \right] \leq \bbound^2$, for which we commence by using the tower rule.
    \begin{align*}
        \E \left[ \|\blhat_t\|_{R, \w_t}^2 \right] &= \E_{\filobs_t} \left[ \E_{\a_t} \left[ \|\blhat_t\|_{R, \w_t}^2 | \filobs_{t} \right] \right]~.
    \end{align*}
    Next we consider $\E_{\action_t} \left[ \|\blhat_t\|_{R, \w_t}^2 | \filobs_{t} \right]$ in isolation
    \begin{align*}
        \E_{\a_t} \left[ \|\blhat_t\|_{R, \w_t}^2 | \filobs_{t} \right] &= \E_{\a_t} \left[ \sumdim \left( \frac{\action_t(i) \bell_t(i)}{\w_t(i)} \right)^2 \left(\nabla^{2} R(\w_t)\right)^{-1}(i, i) | \filobs_{t} \right]\\
        &= \sumdim \frac{\bell_t(i)^2}{\w_t(i)}\frac{\eta \gamma \w_t^2(i)}{\eta + \gamma \w_t(i)}
        \leq \underbrace{\eta \actiondim}_{\bbound^2}~, \numberthis \label{eqn:CMABbbound}
    \end{align*}
    where we also used that $\E_{\action_t}[\action_t] = \w_t$, and $\eta + \gamma \w(i) \geq \gamma \w(i)$.
    Next is showing that $\|\blhat_t\|_{R, \w_t} \leq \frac{1}{64(1 + \dmax)}$:
    \begin{align*}
        \|\blhat_t\|_{R, \w(\hat{L}_t)} = \sqrt{\sumdim \left( \frac{\action_t(i) \bell_t(i)}{\w_t(i)} \right)^2 \frac{\eta \gamma \w_t^2(i)}{\eta + \gamma \w_t(i)}}
        \leq \sqrt{ \gamma \combm }
        \leq \frac{1}{64 (1 + \dmax)}~,
    \end{align*}
    where we used that $\gamma \leq \frac{1}{64^2 \combm (1 + \dmax)^2}$. Finally we show that $\blhat_\tau$ and $\blhat_{\tau'}$ are independent for all $\tau, \tau' \in \miss_t$ where $\tau' \neq \tau$. Recall that $\blhat_\tau(i) = \frac{\action_\tau(i) \bell_\tau(i)}{\w_\tau(i)}$, for all $i$. Conditioned on the observed history $\fil_t$, the only random element of $\blhat_{\tau'}$ is $\action_{\tau'} \sim \p_{\tau'}$. Since $\blhat_\tau$ can not be used in round $t$ to compute $\p_{\tau'}$ we have that $\blhat_{\tau'}$ is independent of $\blhat_\tau$. We conclude that
    \begin{align*}
        &\E\Big[(\blhat_\tau - \bell_\tau)^\top\big(\nabla^2 R(\w_t)\big)^{-1} (\blhat_{\tau'} - \bell_{\tau'})\,\Big|\, \fil_t\Big]\\
        &= \E_{\blhat_{\tau}}\Big[\E\Big[(\blhat_\tau - \bell_\tau)^\top\big(\nabla^2 R(\w_t)\big)^{-1} (\blhat_{\tau'} - \bell_{\tau'})\,\Big|\, \fil_t, \blhat_{\tau}\Big]\Big] = 0
    \end{align*}
    where we used that $\blhat_{\tau'}$ is and unbiased estimator of $\bell_{\tau'}$.
    Now we are able to apply Corollary~\ref{cor:main}, from which it follows that for any $\u \in \domainw$
    \begin{align}
        \E \left[ \sumT (\action_t - \u)^\top \bell_t \right] &\leq R(\u) - R(\w_1) + 16 \eta \actiondim T + 16 \eta \combm \totaldelay~, \label{eqn:CMABcor2}
    \end{align}
    by substituting $\fbound$ from equation~\eqref{eqn:CMABfbound} and ${\bbound^2}$ from equation~\eqref{eqn:CMABbbound}.
    The next step is finding an appropriate bound on $R$. We can do this for all $\u \in \domainw$ for the negative entropy component as follows
    \begin{align*}
        - \sumdim \u(i) \ln \u(i) &= \|\u\|_1 \sumdim \frac{\u(i)}{\|\u\|_1} \ln \frac{1}{\w(i)}\\
        &\leq \|\u\|_1 \ln \l( \sumdim \frac{\u(i)}{\|\u\|_1} \frac{1}{\u(i)} \r)\\
        &\leq \|\u\|_1 \ln \l( \frac{\actiondim}{\|\u\|_1} \r) + \|\u\|_1
        \leq \combm \l( 1 + \ln \l( \frac{\actiondim}{\combm} \r)\r), \numberthis \label{eqn:CMABentropybound}
    \end{align*}
    where we used Jensen's inequality in the second step and the fact that $x \ln(\tfrac{K}{x}) + x$ is in increasing on $x \in [1, \actiondim]$ in the last inequality. The negative logarithm component however is unbounded and tends to infinity when any element of $\u$ tends to 0. Thus we cannot compare to $\a^*$ directly, which might lie on the boundary on $\domainw$. Instead we define $\u = (1 - \combslider)\action^* + \combslider \w_1$ for an $\combslider \in [0, 1]$. $\combslider$ now acts as a trade-off between an upper bound on the regularizer and an additional bias-like term that stems from comparing $\a^*$ to $\u$ in terms of pseudo-regret. We now bound the negative logarithm of $\u$ by simply using $(1 - \combslider)\action^*(i) \geq 0$
    \begin{align*}
        \ln(\w_1(i))-\ln(\u(i)) = \ln\l(\frac{\w_1(i)}{(1 - \combslider)\action^*(i) + \combslider\w_1(i)}\r)
        \leq \ln\l( \frac{1}{\combslider}\r). \numberthis \label{eqn:CMABneglogbound}
    \end{align*}
    Combining equation~\eqref{eqn:CMABentropybound} and equation~\eqref{eqn:CMABneglogbound} allows us to bound $R(\u) - R(\w_1)$
    \begin{align*}
        R(\u) - R(\w_1) &= \sumdim \Big(\frac{\u(i)}{\eta}\ln(\u(i)) - \frac{1}{\gamma}\ln(\u(i))\Big)
        - \sumdim \Big(\frac{\w_1(i)}{\eta}\ln(\w_1(i)) - \frac{1}{\gamma}\ln(\w_1(i))\Big)\\
        &\leq \frac{ \combm \l( 1 + \ln \l( \frac{\actiondim}{\combm} \r)\r)}{\eta} + \frac{\actiondim \ln\l( \frac{1}{\combslider}\r)}{\gamma}~, \numberthis \label{eqn:CMABregbound}
    \end{align*}
    where we applied that $\ln(x) \leq 0$ for all $x \in (0, 1)$.
    %We introduce the Minkowsky function  given by 
    %\begin{align*}
        %\pi_\y(\x) = \inf\{t > 0: \y + t^{-1}(\x - \y) \in \domainw\}.
    %\end{align*}
    %We examine $\pi_{\w_1}(\u)$.
    %\begin{align*}
        %\pi_{\w_1}(\u) &= \inf\{t > 0: \w_1 + t^{-1}(\u - \w_1) \in \domainw\}\\
        %&= \inf\{t > 0: \w_1 + \frac{(1 - \combslider)}{t}(\action^* - \w_1) \in \domainw\}\\
        %&\leq (1 - \combslider)
    %\end{align*}
    %Allowing us to conclude that
    %\begin{align*}
        %R(\u) - R(\w_1) \leq \vartheta \ln\frac{1}{1 - \pi_{\w_1}(\u)} \leq \l( \frac{2}{\eta} + \frac{\eta}{\gamma} \r) \ln \frac{1}{\combslider},
    %\end{align*}
    %where we used $R(\u) - R(\w_1) \leq \ln\frac{1}{1 - \pi_{\w_1}(\u)}$, $\pi_\y(\x) \in [0, 1)$ for $\x, \y \in \domainw \setminus \partial \domainw$ as shown in , and that $R$ is a self-concordant barrier with parameter $\vartheta = \frac{2}{\eta} + \frac{\eta}{\gamma}$, as shown in Lemma~\ref{lem:negentneglog}.
    To finish the proof, we start from the regret
    \begin{align*}
        \regret_T &= \E \left[ \sumT (\action_t - \action^*)^\top \bell_t \right]
        = \E \left[ \sumT (\action_t - \u)^\top \bell_t \right] + \combslider \E \left[ \sumT (\w_1 - \action^*)^\top \bell_t \right]\\
        &\leq \E \left[ \sumT (\action_t - \u)^\top \bell_t \right] + 2 \combslider \combm T~,
    \end{align*}
    where we bound $(\w_1 - \action^*)^\top \bell_t \leq 2 \combm$ in the inequality. We continue by using equation~\eqref{eqn:CMABcor2} and equation~\eqref{eqn:CMABregbound}
    \begin{align*}
        \E \left[ \sumT (\action_t - \u)^\top \bell_t \right] + 2 \combslider \combm T &\leq R(\u) - R(\w_1) + 16 \eta \actiondim T + 16 \eta \combm \totaldelay + 2 \combslider \combm T\\
        &\leq \frac{ \combm \l( 1 + \ln \l( \frac{\actiondim}{\combm} \r)\r)}{\eta} + \frac{\actiondim \ln\l( \frac{1}{\combslider}\r)}{\gamma} + 16 \eta \actiondim T + 16 \eta \combm \totaldelay + 2 \combslider \combm T \\
        & =  \frac{ \combm \l( 1 + \ln \l( \frac{\actiondim}{\combm} \r)\r)}{\eta} + \frac{\actiondim \ln\l( T \r)}{\gamma} + 16 \eta \actiondim T + 16 \eta \combm \totaldelay + 2 \combm~,
    \end{align*}
    where in the equality we set $\combslider = \frac{1}{T}$.
    Substituting
    \[
    \gamma = \frac{1}{64^2 \combm (1 + \dmax)^2}\qquad\text{and}\qquad\eta = \min \l\{ \frac{1}{16^2 \combm \dmax^2}, \frac{\sqrt{\combm \l( 1 + \ln \l( \frac{\actiondim}{\combm} \r)\r) }}{4 \sqrt{(\combm \totaldelay + \actiondim T) } } \r\}
    \]
    yields
    \begin{align*}
        \regret_T & \leq 8\sqrt{\combm \l( 1 + \ln \l( \frac{\actiondim}{\combm} \r)\r)(\actiondim T + \combm \totaldelay)}\\
        &\quad+ 64^2 \combm (1 + \dmax)^2 \actiondim \ln\l( T \r) + 2^8 \combm^2 \dmax^2 \l( 1 + \ln \l( \frac{\actiondim}{\combm} \r)\r)~.
    \end{align*}
    
    %Lastly we set $\combslider = \frac{1}{T}$ to its optimal value at $\combslider = \min\{1, \frac{\actiondim}{2 \gamma \combm T}\}$, for which we can use full information as it is not an actual parameter of the algorithm. 
    
    % The first statement then follows by plugging in $\combslider$
    % \begin{align*}
    %     \regret_T &\leq \frac{\actiondim}{\gamma} + \frac{\combm \ln \l( \frac{\actiondim}{\combm} \r)}{\eta} + \frac{\actiondim \ln\l(\actiondim + 2 \gamma \combm T  \r)}{\gamma} + 12 \eta \actiondim T + 12 \eta \combm \totaldelay.
    % \end{align*}
    % By picking
    % \begin{align*}
    %     \gamma &= \frac{1}{64^2 \combm (1 + \dmax)^2}\\
    %     \eta &= \min \l\{ \frac{1}{16^2 \combm \dmax^2}, \frac{\sqrt{\combm \ln \l( \frac{\actiondim}{\combm}\r) }}{\sqrt{ 12 (\combm \totaldelay + \actiondim T) } } \r\},
    % \end{align*}
    % we obtain
    % \begin{align*}
    %     \regret_T &\leq 8 \sqrt{ \combm (\combm \totaldelay + \actiondim T) \ln \l( \frac{\actiondim}{\combm}\r)} + 2 \cdot 64^2 \combm^2 (1 + \dmax)^2 \ln\l(\actiondim + T  \r) + 16^2 \combm^2 \dmax^2 \ln \l( \frac{\actiondim}{\combm} \r)\\
    %     &\leq 8 \combm \sqrt{\totaldelay \ln \l( \frac{\actiondim}{\combm}\r)} + 8 \sqrt{\combm \actiondim T \ln \l( \frac{\actiondim}{\combm}\r)} + (2 \cdot 64^2 + 16^2) \combm^2 (1 + \dmax)^2 \ln\l(\actiondim + T  \r),
    % \end{align*}
    % where we used $2 \sqrt{12} + \frac{1}{\sqrt{12}} \leq 8$ finishing the proof.
\end{proof}

\thLBcombinatorial*
\begin{proof}
    By \citet{audibert2014regret}, we have that any algorithm without delay must suffer at least $\Omega(\sqrt{BKT})$ regret in the combinatorial semi-bandit setting. 
    
    Next, we assume full information feedback, which is easier from the point of view of the algorithm. We take inspiration from \citet[Lemma~3]{langford2009slow}. For simplicity we will assume that $T/d$ is an integer. We divide the $T$ rounds into $T/d$ blocks of $d$ rounds. We take the losses of the lower bound for $B$-sets in \citep[Section~4]{koolen2010hedging}, which states that any algorithm in the full information setting must suffer at least $\Omega(B\sqrt{T'})$ regret after $T'$ rounds. We take the loss of the first round of the lower bound \citep{koolen2010hedging} and copy it $d$ times, which we use as the losses for the first block. We repeat this process for the remaining blocks. Since the algorithm can not respond to the copied losses, we must have that any algorithm must suffer at least $\Omega(d B\sqrt{T/d}) = \Omega(B\sqrt{dT})$ regret, which completes the proof. 
\end{proof}

\newpage
\section{Deferred Proofs for Adversarial MDPs (Section \ref{sec:MDPs})}
\label{appendix: MDPs}

\thmMDPmain*

\begin{proof}
    First, we decompose 
    \begin{align*}
        \regret_T = 
        \sum_{t=1}^T\langle \w^{\pi_{t}} - \w^{\star}, \mdpell_{t} \rangle
        = \underbrace{
        \sum_{t=1}^T \langle\w^{\pi_{t}} - \w_{t}, \mdpell_{t}\rangle
                    }_{\textsc{Error}} 
            +  \underbrace{ \sum_{t=1}^T
            \langle\w_{t} - \tilde{\w}^{\star}, \mdpell_{t}\rangle
                        }_{\textsc{Reg}}
            +  \underbrace{ \sum_{t=1}^T
            \langle \tilde{\w}^{\star} - \w^{\star}, \mdpell_{t}\rangle
                        }_{\textsc{Shift-penalty}},
    \end{align*}
    where, by Lemma~\ref{lem:MDPs known p trancation error}, \textsc{Error} is bounded by $2H$ and $\tilde{w}^{\star} \in \domainw$ exists such that \textsc{Shift-penalty} is bounded by $2H$. For \textsc{Reg} we use Lemma~\ref{lem:main}. Much like in the proof of Lemma~\ref{th: CMAB regret}, $4 \nabla^2 R( \w) \succeq \nabla^2 R( \w') \succeq \frac{1}{4} \nabla^2 R( \w) $ for all $\w \in \domainw$ and $\w' \in \dikin_R(\w, \half)$. For any $t$ and $\w \in \domainw$
    \[
        \| \mdpell_t \|_{R,\w} 
        \leq 
        \sqrt{\eta \sum_{h,s,a} \w_h(s,a) \mdpell_h(s,a)^2} 
        \leq 
        \sqrt{\eta \sum_{h,s,a} \w_h(s,a)} 
        = 
        \sqrt{\eta H} \eqqcolon
        \alpha
        \leq \frac{1}{16 (1+d_{max})}
        ,
    \]
    where the last inequality is by the definition of $\eta$. For any $t$,
    \begin{align*}
        \bbE[\| \hat{\mdpell}_{t} \|_{R, \w_{t}}^{2}] 
        & =
        \eta \bbE\l[\sum_{h,s,a} \w_{t,h}(s,a) \hat{\mdpell}_{t,h}(s,a)^{2}\r] 
        \leq 
        \eta \bbE \l[\sum_{h,s,a} \frac{ \bbE\l[ \bbI\{s_{t,h}=s,a_{t,h}=a\} \mid \calF_{t}\r]}{ \u_{t,h}(s,a)}\r]  
        \\
        & = 
        \eta \bbE\l[ \sum_{h,s,a} \frac{ \w_{h}^{\pi_{t}}(s,a)}{ \u_{t,h}(s,a)}\r]
        \leq 
        \eta HSA 
        \eqqcolon 
        \beta^{2},                         
    \end{align*}
    where the inequalities follow since $\u_{t,h}(s,a) = \max_{\hat{p} \in \calP} \w^{\pi_\tau,\hat{p}}_h(s,a) \ge \max \{ \w_{t,h}(s,a) , \w^{\pi_t}_{h}(s,a) \}$.
    Finally, for all $t$, 
    \begin{align*}
        \|\hat{\mdpell}_{t} \|_{R,\w_{t}} 
        \leq \sqrt{ \gamma \sum_{h,s,a} \w_{t,h}(s,a)^{2} \hat{\mdpell}_{h}(s,a)^{2} }
    \leq \sqrt{ \gamma\sum_{h,s,a} \bbI\{s_{t,h}=s,a_{t,h}=a\}}
    =    \sqrt{\gamma H}
        \leq \frac{1}{64(1+d_{max})},
    \end{align*}
    where the second is since $\u_{t,h}(s,a) \geq \w^{\pi_t}_{h}(s,a)$ and the last is by definition of $\gamma$.
    Thus, applying Lemma \ref{lem:main} with $b_t = \bbE[\hat \mdpell_t - \mdpell_t \mid \calF_t]$, we get
    \begin{align*}
        \textsc{Reg} & \leq \underbrace{R(\tilde w^\star) - R(w_1)}_{\textsc{Penalty}} + 8\eta H S A T + 8 \eta H (T + D)
        \\
        & \quad + \underbrace{\sum_{t=1}^T \bbE[\w(\hat \L_t^m)^\top 
                                (\mdpell_t - \hat \mdpell_t)]}_{\textsc{Bias}_1} 
                + \underbrace{\sum_{t=1}^T \bbE[\tilde {\w^{\star}}^\top (\hat \mdpell_t - \mdpell_t)]}_{\textsc{Bias}_2}
                + 8 \sqrt{\eta H} \underbrace{\sum_{t=1}^T \bbE[ \|\sum_{\tau \in m_t} (\mdpell_\tau - \hat \mdpell_\tau) \|_{R,\w_t} ]}_{\textsc{Drift}}.
    \end{align*}
    Using standard arguments, $\textsc{Penralty} \leq \frac{4 H S^{2} A\log (H S A T)}{\gamma} + \frac{2 H \log(S A T)}{\eta}$ since $\w_1,\tilde{\w}^\star \in \Omega$.
    Recall that by definition $\u_{t,h}(s,a) \geq \w^{\pi_t}_{h}(s,a)$. Thus, $\bbE[\hat \mdpell_{t,h}(s,a) \mid \calF_t ] \leq \mdpell_t$ and $\textsc{Bias}_2 \leq 0$. 
    $\textsc{Bias}_1$ is the bias of the estimator due to the fact that we use an upper confidence bound on the occupancy measure instead of the actual occupancy measure. By Lemma~\ref{lem:MDPs bias1} we have $\textsc{Bias}_1 \leq 8 H^2 S$.
    Finally, for the $\textsc{Drift}$ term, for each $t$,
    \begin{align}
\nonumber
    \bbE & \|\sum_{\tau  \in m_{t}}(\mdpell_{\tau} - \hat{\mdpell}_{\tau})\|_{R,\w_{t}} 
                  =\bbE \sqrt{(\sum_{\tau \in m_{t}}\mdpell_{\tau} - \hat{\mdpell}_{\tau}) \nabla^{-2}R(\w_{t}) (\sum_{\tau \in m_{t}}\mdpell_{\tau} - \hat{\mdpell}_{\tau})}
                \\
                \nonumber
                & =\bbE \sqrt{  \sum_{\tau \in m_{t}} \|(\mdpell_{\tau} - \hat{\mdpell}_{\tau})\|_{R,\w_{t}}^{2} 
                    + \sum_{\tau \in m_{t}} \sum_{\tau' \in m_{t} \backslash \{\tau\}}
                    (\mdpell_{\tau} - \hat{\mdpell}_{\tau}) \nabla^{-2}R(\w_{t}) 
                    (\mdpell_{\tau'} - \hat{\mdpell}_{\tau'})}
                \\
                & \leq \sqrt{ \sum_{\tau \in m_{t}} \bbE \|(\mdpell_{\tau} - \hat{\mdpell}_{\tau})\|_{R,\w_{t}}^2 }
                    + \bbE \sqrt{\sum_{\tau \in m_{t}} \sum_{\tau' \in m_{t} \backslash \{\tau\}} 
                        (\mdpell_{\tau} - \hat{\mdpell}_{\tau}) \nabla^{-2}R(\w_{t}) 
                        (\mdpell_{\tau'} - \hat{\mdpell}_{\tau'})},
        \nonumber
    \end{align}
    where the inequality is by $\sqrt{a+b} \leq \sqrt{a} + \sqrt{b}$ and Jensen.
    For the first term, by Lemma~\ref{lem:induction},
    \begin{align*}
	\sum_{\tau \in m_{t}} \bbE\| (\mdpell_{\tau}-\hat{\mdpell}_{\tau}) \|_{R,\w_{t}}^{2} 
                & \leq 4 \sum_{\tau\in m_{t}} \bbE\| (\mdpell_{\tau}-\hat{\mdpell}_{\tau}) \|_{R,\w_{\tau}}^{2}
                    \leq 4 \sum_{\tau\in m_{t}} \bbE \|\mdpell_{\tau}\|_{R,\w_{\tau}}^{2}
                       + 4 \sum_{\tau\in m_{t}} \bbE\|\hat{\mdpell}_{\tau} \|_{R,\w_{\tau}}^{2} \\
                &  \leq 4|m_{t}|(\alpha^{2}+\beta^{2})\leq 8\eta HSA|m_{t}|.
\end{align*}
    In Lemma~\ref{lem:MDPs bias3} we bound the second term similarly to $\textsc{Bias}_1$ by $4\sqrt{\frac{H^2S}{T}}$ due to the estimator's small bias.
    Overall,
    \begin{align*}
    	\sqrt{\eta H}\cdot & \textsc{Drift} 
                    \leq 8 \eta H\sum_{t=1}^{T}\sqrt{SA|m_{t}|} + 4 \sqrt{ \eta H^{3} S T}
                        \leq 8 \eta H \sum_{t=1}^{T} \sqrt{S^{2}A^{2} + |m_{t}|^{2}}  + 4 \sqrt{ \eta H^{3} S T}
                        \\
                     & \leq 8 \eta H S A T + 8 \eta H\sum_{t=1}^{T} |m_{t}| + 4 \sqrt{ \eta H^{3} S T}
                        \leq 8 \eta H S A T + 8 \eta H D  + 4 \sqrt{ \eta H^{3} S T},                                 
    \end{align*}
    where we used $ab \leq a^2 + b^2$ and $\sum_{t=1}^T |m_t| = D$. 
    Finally, we sum all the different terms.
\end{proof}
\lemmaMDPestimatorbias*
\begin{proof}
\begin{enumerate}
    \item 
    Define $\tilde{p} = \{ \tilde{p}_h:\mdpS \times \mdpA \to \Delta_{\mdpS} \}_{h=1}^H$ by $\tilde{p}_h(s' \mid s,a) = (1- \frac{1}{THSA}) p_h(s' \mid s,a) + \frac{1}{THS^2A}$ and notice that $\tilde{p} \in \calP$ since $|p_h(s' \mid s,a) - \tilde{p}_h(s' \mid s,a)| \le \frac{1}{THSA}$.
    Next, let $\pi_u$ be the uniformly random policy, and define $\tilde{\w} = (1 - \frac{1}{T}) \w + \frac{1}{T} \w^{\pi_u,\tilde{p}}$.
    It holds that $\tilde{\w} \in \Delta(\calP)$ because $\Delta(\calP)$ is a convex set.
    Moreover, notice that $\w^{\pi_u,\tilde{p}}_h(s,a,s') \ge \frac{1}{(T H S^2 A)^2 A}$ which implies that $\tilde{\w}_h(s,a,s') \ge \frac{1}{T^3 H^2 S^4 A^2}$.
    Thus, $\tilde{\w} \in \domainw$.
    Finally,
    \begin{align*}
        \Vert \w -  \tilde{\w} \Vert_1
        & =
        \sum_{h,s,a,s'} \left| \w_h(s,a,s') -  \tilde{\w}_h(s,a,s') \right|
        \\
        & =
        \sum_{h,s,a,s'} \left| \frac{1}{T} \w_h(s,a,s') - \frac{1}{T} \w^{\pi_u,\tilde{p}}_h(s,a,s') \right|
        \\
        & \le
        \frac{1}{T} \sum_{h,s,a,s'} \w_h(s,a,s') + \frac{1}{T} \sum_{h,s,a,s'} \w^{\pi_u,\tilde{p}}_h(s,a,s')
        =
        \frac{2H}{T}.
    \end{align*}

    \item
    Define loss function $\tilde{\mdpell}_h(s,a) = \text{sign}(\w^{\pi}_h(s,a) - \w_{t,h}(s,a) )$ and note that $\| \w^{\pi} - \w \|_1 = V^{\pi,p,\tilde{\mdpell}}_1(\sinit) - V^{\pi, \hat p ,\tilde{\mdpell}}_1(\sinit)$ for some $\hat p \in \calP$. Combining Lemma~\ref{lem: value diff}\footnote{
    We note that \cite{even2009online} (see also \cite{shani2020optimistic}) apply the value difference lemma (Lemma \ref{lem: value diff}) on positive losses (or rewards), where here we use loss function supported in $[-1,1]$. However, the proof of the lemma in fact holds for any loss functions $\tilde{\mdpell}, \mdpell \subseteq \bbR^{H S A}$.
    } and the fact that $\| p - \hat{p} \|_\infty \leq \frac{1}{THSA}$ proves that $\| \w^{\pi} - \w \|_1 \leq \frac{2H}{T}$. 
    Now, let $\hat p^{h,s}$ be the transition function that corresponds to $\u_h(s)$. We have that, $\|\hat p^{h,s} - \hat p\|_\infty \leq \|\hat p^{h,s} -  p\|_\infty + \| p - \hat p\|_\infty \leq \frac{2}{THSA}$. Thus, using the same argument as the above,
    \[
        \| \u - \w \|_1 \leq \sum_{h,s}\| \w^{\pi,\hat p^{h,s}} - \w \|_1 \leq \frac{4 H^2 S}{T}.
    \]
    \end{enumerate}
\end{proof}

\begin{lemma}[$\textsc{Bias}_1$]
    \label{lem:MDPs bias1}
    When running Delayed FTRL for adversarial MDPs we have,
    \[
        \textsc{Bias}_1 = \sum_{t=1}^T \bbE[\w(\hat \L_t^m)^\top 
                                (\mdpell_t - \hat \mdpell_t)] \leq 8 H^2 S.
    \]
\end{lemma}
\begin{proof}
    Let $\calG_t$ be the history of all episodes in $[t-1]$, and note that $\w_t,\u_t$ and $\w(\hat \L _t^m)$ are all determined by $\calG_t$. Therefore,
    \begin{align*}
        \textsc{Bias}_{1} & = \bbE\l[ \sum_{t,h,s,a} \w(\hat{\L}_{t}^{m})_{h}(s,a)(\mdpell_{t,h}(s,a) 
                                - \bbE[\hat{\mdpell}_{t,h}(s,a)\mid{\calG}_{t}]) \r] \\
                          & = \bbE\l[ \sum_{t,h,s,a} \w(\hat{\L}_{t}^{m})_{h}(s,a) \mdpell_{t,h}(s,a) \l( 1 - \frac{\w_{h}^{\pi_{t}}(s,a)}{ \u_{t,h}(s,a)} \r) \r] \\
                          & \leq \bbE\l[ \sum_{t,h,s,a} \w(\hat{\L}_{t}^{m})_{h}(s,a) 
                                \frac{| \u_{t,h}(s,a) - \w_{h}^{\pi_{t}}(s,a) | }{\u_{t,h}(s,a)} \r]
    \end{align*}
    Now, as in the proof of Lemma \ref{lem:main}, $\w(\hat \L_t^m) \in \dikin_R(\w_t, \frac12)$. Thus, by Lemma \ref{lem:dikin implies factor 2}, $\w(\hat{\L}_{t}^{m})_{h}(s,a) \leq 2 \w_{t,h}(s,a) \leq 2\u_{t,h}(s,a)$. Therefore,
    \[
        \textsc{Bias}_{1} \leq 2 \sum_{t=1}^T \bbE[ \| \u_t - \w^{\pi_t} \|_1 ] \leq 8 H^2 S
    \]
    where the last is by article 2 in Lemma \ref{lem:MDPs known p trancation error}. 
\end{proof}

\begin{lemma}
    When running Delayed FTRL for adversarial MDPs, for any $t$,
    \label{lem:MDPs bias3}
    \[
        \bbE \sqrt{\sum_{\tau \in m_{t}} \sum_{\tau' \in m_{t} \backslash \{\tau\}} 
                        (\mdpell_{\tau} - \hat{\mdpell}_{\tau}) \nabla^{-2}R(\w_{t}) 
                        (\mdpell_{\tau'} - \hat{\mdpell}_{\tau'})} 
        \leq 
        4 \sqrt{\frac{H^{2} S }{T}}   .
    \]
\end{lemma}
\begin{proof}
    We first apply the law of total expectation and Jensen
    inequality,
    \begin{align*}
    & \bbE \sqrt{\sum_{\tau \in m_{t}} \sum_{\tau' \in m_{t} \backslash \{\tau\}} 
                        (\mdpell_{\tau} - \hat{\mdpell}_{\tau}) \nabla^{-2}R(\w_{t}) 
                        (\mdpell_{\tau'} - \hat{\mdpell}_{\tau'})} 
    \\	
      & 
      \leq \bbE \sqrt{\eta\sum_{\tau\in m_{t}} \sum_{\tau'\in m_{t} \slash\{\tau\}} \sum_{h,s,a} \w_{t,h}(s,a)
                        (\mdpell_{\tau,h}(s,a) - \bbE[\hat{\mdpell}_{\tau,h}(s,a) \mid \calG_{t}])
                        (\mdpell_{\tau',h}(s,a) - \bbE[\hat{\mdpell}_{\tau',h}(s,a) \mid \calG_{t}])} 
    \end{align*}
    Now, let $\calG_t$ be the history of all episodes in $[t-1]$, and note that $\mdpell_{\tau',h}(s,a)-\bbE[\hat{\mdpell}_{\tau',h}(s,a)\mid\calG_{t}]\in[0,1]$.
    Thus, we can further bound the above by,
    \begin{align*}
    	& \bbE \sqrt{\eta\sum_{\tau\in m_{t}} \sum_{\tau'\in m_{t} \slash\{\tau\}} \sum_{h,s,a} \w_{t,h}(s,a)
                    (\mdpell_{\tau,h}(s,a) - \bbE[\hat{\mdpell}_{\tau,h}(s,a) \mid \calG_{t}])} 
                \\
    	    &\qquad \leq \bbE \sqrt{\eta |m_{t}| \sum_{\tau\in m_{t}} \sum_{h,s,a} \w_{t,h}(s,a) \mdpell_{\tau,h}(s,a) \frac{\u_{\tau,h}(s,a) - \w_{\tau,h}(s,a)}{\u_{\tau,h}(s,a)}}
                \\
    	    &\qquad \leq 2 \bbE \sqrt{ \eta |m_{t}| \sum_{\tau\in m_{t}} \sum_{h,s,a} \w_{\tau,h}(s,a)   
                    \frac{\u_{\tau,h}(s,a) - \w_{\tau,h}(s,a)}{\u_{\tau,h}(s,a)}}
                \\
    	    &\qquad \leq 2 \bbE \sqrt{\eta |m_{t}| \sum_{\tau\in m_{t}} \sum_{h,s,a} 
                    \u_{\tau,h}(s,a) - \w_{\tau,h}(s,a)}
                \\
    	    &\qquad = 2\bbE \sqrt{\eta |m_{t}| \sum_{\tau\in m_{t}} \|\u_{\tau} - \w_{\tau}\|_{1}}
                \\
    	    &\qquad \leq 4 \sqrt{\eta |m_{t}|^2 \frac{H^{2} S }{T}}      
                    \leq  4 \sqrt{\frac{H^{2} S }{T}}  , 
    \end{align*}
    where the third inequality is by Lemma \ref{lem:dikin implies factor 2}, the forth is since $\w_{\tau,h}(s,a) \leq \u_{\tau,h}(s,a)$, the fifth inequality is by Lemma~\ref{lem:MDPs known p trancation error}, and the last is since $\eta\leq \frac{1}{d_{max}^2} \leq \frac{1}{|m_t|^2}$.
\end{proof}

\newpage
\section{Deferred Proofs for Linear Bandits (Section~\ref{sec:linear-bandits})}\label{sec:appendix-linear-bandits}

\linbanth*
\begin{proof}
    We start by verifying the assumptions of Corollary~\ref{cor:main}. Using $\E[\v_t] = \0$ and $\E[\v_t\v_t^\top] = \frac{1}{K}\I$ we see that $\E[\blhat_t] = \bell_t$. 
    Observe that the distribution of $\blhat_{\tau'}$ is fully determined given $\mathcal{F}_t$ because $\mathcal{F}_{\tau'} \subseteq \mathcal{F}_t$. Furthermore, since $\blhat_{\tau}$ can not be used in round $\tau'$ because $\tau$ is not available in round $t$ due to the delay, we must have that $\blhat_{\tau'}$ is independent of $\blhat_{\tau}$. Thus, by the tower rule
    \begin{align*}
        & \E\big[(\blhat_\tau - \bell_\tau)^\top\big(\nabla^2 R(\w_t)\big)^{-1} (\blhat_{\tau'} - \bell_{\tau'})\,|\,\mathcal{F}_t\big] \\
        & = \E_{\blhat_{\tau}}\big[\E\big[(\blhat_\tau - \bell_\tau)^\top\big(\nabla^2 R(\w_t)\big)^{-1} (\blhat_{\tau'} - \bell_{\tau'}) \,|\,\mathcal{F}_t, \blhat_{\tau}\big]\big] = 0~, 
    \end{align*}
    where we used that $\E[\blhat_{\tau'}|\mathcal{F}_t] = \E[\blhat_{\tau}|\mathcal{F}_t, \blhat_{\tau}] = \bell_t$.
    %
    % Due to the delay, for $\tau \neq \tau'$ such that $\tau, \tau' \in \miss_t$ we have that 
    % %
    % \[
    % \E\big[(\blhat_\tau - \bell_\tau)^\top\big(\nabla^2 R(\w_t)\big)^{-1} (\blhat_{\tau'} - \bell_{\tau'})\,|\,\mathcal{F}_t\big] = 0~.
    % \]
    % %
    Next, observe that because $\nabla^2 R(\w) \succeq \frac{1}{\gamma} \nabla^2 \Psi(\w)$ we have that
    \begin{align*}
        \|\blhat_t\|_{R, \w_t}^2 \leq \gamma K^2 (\bell_t^\top \action_t)^2 \v_t^\top \v_t =  \gamma K^2 (\bell_t^\top \action_t)^2~.
    \end{align*}
    Since $\|\bell_t\|_2 \leq 1$ and $\domainw \subseteq \ball(B)$, we have that $(\bell_t^\top \action_t)^2 \leq B^2$ and thus 
    \begin{align*}
        \|\blhat_t\|_{R, \w_t} \leq \underbrace{\sqrt{\gamma (BK)^2}}_{\bbound} \leq \frac{1}{64(1+\dmax)}~,
    \end{align*}
    where the last inequality is because $\gamma \leq (64BK(1+\dmax))^{-2}$. 
    Because $\nabla^2 R(\w) \succeq \frac{1}{\eta} \I$ and $\|\bell_t\|_2 \leq 1$ we have that 
    \begin{align*}
        \|\bell_t\|_{R, \w} \leq \underbrace{\sqrt{\eta}}_{\fbound} \leq \frac{1}{16\dmax}~,
    \end{align*}
    where the last inequality is because $\eta \leq (16\dmax )^{-2}$. 
    The final assumption to verify is $4 \nabla^2 R( \w) \succeq \nabla^2 R( \w') \succeq \frac{1}{4} \nabla^2 R( \w)$ for all $\w \in \domainw$ and $\w' \in \dikin_R(\w, \half)$, which is immediate due to equation~\eqref{eq:scbstability}.

    Pick $\tilde{\u} \in \domainw$ and $\delta > 0$. Set $\u = \frac{\tilde{\u} - \w_1}{1+\delta} + \w_1 \in \domainw_\delta$. Using equation~\eqref{eq:scbbias} and $\domainw \subseteq \ball(B)$ we have that
    \begin{align}
    \label{eq:thisbound}
        R(\u) - R(\w_1) \leq \frac{B^2}{\eta} + \frac{\nu}{\gamma} \ln\Big(\frac{1+\delta}{\delta}\Big)~.
    \end{align}
    Furthermore, by using that $\domainw \in \ball(B)$ and $\|\bell_t\|_2 \leq 1$ , we have that
    \begin{align*}
        \sumT (\tilde{\u} - \u)^\top \bell_t = \sumT \left(1 - \frac{1}{1+\delta}\right)(\tilde{\u} - \w_1)^\top\bell_t
        \leq 2TB\left(\frac{\delta}{1+\delta}\right)~.
    \end{align*}
    Thus, by setting $\delta = \frac{1}{\sqrt{T}}$ and then applying Corollary~\ref{cor:main} with $\beta^2 = \gamma (BK)^2$  we obtain
    \begin{align*}
        \E&\left[\sumT (\action_t - \tilde{\u})^\top \bell_t \right] \leq \E\left[\sumT (\action_t - \u)^\top \bell_t \right] + 2B\sqrt{T} \\
        & \leq R(\u) - R(\w_1) + 16 \gamma (BK)^2 + 16 \eta \sumT |\miss_t|  + 2B\sqrt{T}\\
        & \leq \frac{B^2}{\eta} + \frac{\nu}{\gamma} \ln\big(1 + \sqrt{T}\big) + 16 \gamma (BK)^2 T + 16 \eta  D  + 2B\sqrt{T} \tag{using \eqref{eq:thisbound}} \\
        & \leq 8 B\sqrt{D} + 256 \dmax^2  \\
        &\quad + 8 BK\sqrt{\nu T\ln\big(1 + \sqrt{T}\big)} + \big(64BK(1+\dmax)\big)^2 \nu \ln\big(1 + \sqrt{T}\big) + 2 B\sqrt{T}~,
    \end{align*}
    where in the last step we used our choices for $\eta$ and $\gamma$. 
    % $\eta = \min \Big\{\sqrt{\frac{B^2}{16 L^2 D}}, \frac{1}{(16\dmax L)^2}\Big\}$ and that $\gamma = \min\Big\{\sqrt{\frac{\nu \ln\big(1 + \sqrt{T}\big)}{16 (BLK)^2 T}}, \frac{1}{(64BLK(1+\dmax))^2}\Big\}$.
\end{proof}

\newpage
\section{Doubling with Delayed Feedback}
\label{sec:doubling}
In this section we show how to handle unknown problem parameters. For simplicity of presentation we assume that only $d_{max}$ is unknown. The case of unknown $T$ and $D$ can be done in a similar fashion (e.g., see \cite{bistritz2019online,lancewicki2020learning}).

\begin{algorithm}[t]
    \caption{Doubling procedure}
    \label{alg:doubling}
    \begin{algorithmic}
    \STATE \textbf{Input:} $T,D$ and algorithm $ALG$ (for known $T,D$ and $d_{max}$).
    \STATE Set epoch index $e=1$ and initialize $ALG$ with $T,D$ and $2^{e}$ as $d_{max}$.
    \FOR{$t = 1,...,T$}
        \IF{$\max_{j\in o_{t}}d_{j} > 2^e$}
            \STATE Start a new epoch $e=e+1$, and re-initiate $ALG$ with $T,D$ and $2^{e}$ as $d_{max}$.
        \ENDIF
        \STATE Play according to $ALG$.
    \ENDFOR
\end{algorithmic}
\end{algorithm}

\begin{theorem}
    Let $ALG$ be an algorithm for known $T,D$ and $d_{max}$ and assume that $ALG$ guarantees regret of $R_{T,D}(d_{max})$ whenever initiated properly. Then, running Algorithm~\ref{alg:doubling} with unknown $d_{max}$ guarantees regret,
    \[
        \regret_T \leq 
            2R_{T,D}(2d_{max}) \log T
                + 2Md_{max} \log T ,
    \]
    where $M = \max_{t\in[T], \action, \tilde \action \in \actionset}(\action - \tilde \action)^{\top} \bell_{t}$ is the maximal regret per round (e.g., in Section~\ref{sec:MDPs}, $M \leq H$).
\end{theorem}

\begin{proof}
    Let ${\cal T}_{e}=\{t:2^{e-1}\leq\max_{j\in o_{t}}d_{j}\leq2^{e}\}$ be the set of indices of epoch $e$,
    and let $\tilde{{\cal T}}_{e}=\{t\in{\cal T}_{e}:d_{t}\leq2^{e}\}$ be the indices of epoch $e$ with delay $\leq 2^e$ . The
    regret in rounds $t\in\tilde{{\cal T}}_{e}$ is at most $R_{T,D}(2^{e})\leq R_{T,D}(2d_{max})$
    since the maximal delay in these rounds is indeed bounded by $2^{e}$. In
    addition, the regret in ${\cal T}_{e} \backslash \tilde{{\cal T}}_{e}$
    is at most $Md_{max}$ since $|{\cal T}_{e} \backslash \tilde{{\cal T}}_{e}|\leq d_{max}$.
    Thus, the total regret in epoch $e$ is at most,
    \[
        \underbrace{R_{T,D}(2d_{max})
                    }_{\text{Regret in ${\cal \tilde{T}}_{e}$ }}
            +\underbrace{Md_{max}
                        }_{\text{Regret in ${\cal T}_{e}\backslash \tilde{\cal T}_e$} }.
    \]
    Finally, the total number of epochs is at most $\log d_{max}+1\leq2\log T$
    and thus, the total regret is bounded by,
    \[
    	\regret_T \leq 2R_{T,D}(2d_{max})\log T + 2Md_{max}\log T.
    \]
\end{proof}

\newpage
\section{Auxiliary Lemmas}\label{app:auxiliary}
\begin{lemma}[Value Difference Lemma \cite{even2009online}]
\label{lem: value diff}
    For any two triplets $(\pi,p,\mdpell)$ and $(\tilde \pi, \tilde p,\tilde \mdpell)$ of policy, transition and cost function,
    \begin{align*}
        V^{\pi,p,\mdpell}_1(\sinit) - V^{\tilde \pi, \tilde p,\tilde \mdpell}_1(\sinit) 
            & = \sum_{h=1}^H \bbE_{s\sim \w_h^{\tilde \pi, \tilde p}} \Big[ \l\langle \pi_h(\cdot \mid s) - \tilde \pi_h(\cdot \mid s), Q_h^{\pi,p,\mdpell}(s,\cdot) \r\rangle \Big]
            \\
            & \quad + \sum_{h=1}^H \bbE_{s,a\sim \w_h^{\tilde \pi, \tilde p}} \Big[  \mdpell(s,a) - \tilde\mdpell(s,a) \Big]
            \\
            & \quad + \sum_{h=1}^H \bbE_{s,a\sim \w_h^{\tilde \pi, \tilde p}} \Big[ \l\langle p_h(\cdot \mid s) - \tilde p_h(\cdot \mid s), V_{h + 1}^{\pi,p,\mdpell} \r\rangle \Big]
    \end{align*}
    
\end{lemma}

\begin{lemma}
    \label{lem:dikin implies factor 2}
    Let $\domainw \subseteq \{\w  \in \bbR^n :\forall i\in [n],\,\, \w(i) > 0 \}$. Let $R:\domainw \to \bbR$ be some twice-differentiable convex function, and let $\phi(\w) = -\frac{1}{\gamma} \sum_{i=1}^n \log \w(i) $ be the log barrier with $\gamma \in (0,1)$. Assume that for any $\w\in\domainw$, $\nabla^2 R(\w) \succeq \nabla^2 \phi(\w)$. Then for any $\w' \in \dikin_R(\w,\frac{1}{2})$,
    \[
        \forall i\in[n],
        \qquad
        \frac{1}{2} w(i)\leq w'(i) \leq  2 w(i)    .
    \]
\end{lemma}

\begin{proof}
    Since $\nabla^2 R(\w) \succeq \nabla^2 \phi(\w)$, for any $\w' \in \dikin_R(\w,\frac{1}{2})$, 
    \[
        (\|\w' - \w\|_{\phi,\w}^{*})^{2}
        \leq (\|\w' - \w\|_{R,\w}^{*})^{2}
        \leq \frac{1}{4}.
    \]
    On the other hand,
    \[
        (\|w'-w\|_{\phi,\w}^{*})^{2}
        = \sum_{j=1}^{n}\frac{(\w'(j) - \w(j))^{2}}{\gamma\w(j)^{2}}
        \geq \frac{(\w'(i) - \w(i))^{2}}{\gamma \w(i)^{2}}
        \geq \frac{(\w'(i) - \w(i))^{2}}{\w(i)^{2}}.
    \]
    Thus, $|\w'(i)-\w(i)| \leq \frac{1}{2} \w(i)$ which implies that $\frac{1}{2} \w(i)\leq \w'(i) \leq  2 \w(i) $.
\end{proof}

\end{document}